\documentclass[9.5bp,journal,compsoc]{IEEEtran}
\usepackage{epsfig}
\usepackage{graphicx}
\usepackage[pagebackref=false,breaklinks=true,colorlinks,linkcolor={black},citecolor={black},urlcolor={blue},bookmarks=false]{hyperref}
\usepackage{cite}
\usepackage{microtype}
\usepackage{subfigure}
\usepackage{url}

\usepackage{array}
\usepackage[normalem]{ulem}
\usepackage{ulem}
\usepackage{amsmath}
\usepackage{graphicx,subfigure}
\usepackage{longtable}
\usepackage{rotating}
\usepackage{multirow}
\usepackage{makecell}
\usepackage{bm}
\usepackage{float}
\usepackage{flushend}
\usepackage{xcolor}
\usepackage{soul}
\usepackage{booktabs}
\usepackage{threeparttable}
\usepackage{color}
\usepackage{algorithmic}
\usepackage{tabularx}
\usepackage{amsfonts}
\usepackage{epstopdf}
\usepackage{latexsym}
\usepackage{amssymb}
\usepackage{footnote}
\usepackage{pifont}
\usepackage[ruled]{algorithm2e}

\newtheorem{theorem}{Theorem}
\newtheorem{lemma}{Lemma}
\newtheorem{proof}{Proof}
\newtheorem{definition}{Definition}
\newtheorem{remark}{Remark}
\renewcommand{\IEEEQED}{\IEEEQEDopen}

\usepackage{threeparttable}

\def\Sp{{\scriptsize{\textcircled{{\emph{\tiny{\textbf{Sp}}}}}}}}

\def\QEDclosed{\mbox{\rule[0pt]{1.3ex}{1.3ex}}}

\begin{document}

\title{Rethinking \emph{k}-means from manifold learning perspective}

\author{Quanxue~Gao,  Qianqian~Wang, Han~Lu, Wei~Xia, and~Xinbo~Gao,~\IEEEmembership{Senior Member,~IEEE}
\IEEEcompsocitemizethanks{
\IEEEcompsocthanksitem This work is supported in part by National Natural Science Foundation of China under Grants 62176203; in part by the Open Project Program of State Key Laboratory of Multimodal Artificial Intelligence Systems under Grant 202200035; in part by Natural Science Basic Research Plan in Shaanxi Province (Grant 2020JZ-19), and in part by the Fundamental Research Funds for the Central Universities and the Innovation Fund of Xidian University. (Corresponding author: Quanxue Gao.)\protect

\IEEEcompsocthanksitem Q. Gao, Q. Wang, H. Lu and W. Xia are with the State Key laboratory of Integrated Services Networks, Xidian University, Xi'an 710071, China.\protect

\IEEEcompsocthanksitem X. Gao is with the Chongqing Key Laboratory of Image Cognition, Chongqing University of Posts and Telecommunications, Chongqing 400065, China.\protect}%
}

\markboth{Preprint of IEEE TRANSACTIONS}%
{Shell \MakeLowercase{\textit{Xia et al.}}: Multi-View Clustering via Semi-non-negative Tensor Factorization}

\IEEEtitleabstractindextext{%
\begin{abstract}
Although numerous clustering algorithms have been developed, many existing methods still leverage \emph{k}-means technique to detect clusters of data points. However, the performance of \emph{k}-means heavily depends on the estimation of centers of clusters, which is very difficult to achieve an optimal solution. Another major drawback is that it is sensitive to noise and outlier data. In this paper, from manifold learning perspective, we rethink \emph{k}-means and present a new clustering algorithm which directly detects clusters of data without mean estimation. Specifically, we construct distance matrix between data points by Butterworth filter such that distance between any two data points in the same clusters equals to a small constant, while increasing the distance between other data pairs from different clusters. To well exploit the complementary information embedded in different views, we leverage the tensor Schatten \emph{p}-norm regularization on the 3rd-order tensor which consists of indicator matrices of different views. Finally, an efficient alternating algorithm is derived to optimize our model. The constructed sequence was proved to converge to the stationary KKT point. Extensive experimental results indicate the superiority of our proposed method.
\end{abstract}

}

\maketitle
\IEEEdisplaynontitleabstractindextext
\IEEEpeerreviewmaketitle

\IEEEraisesectionheading{\section{Introduction}\label{sec:introduction}}
\IEEEPARstart{C}{lustering}, which belongs to unsupervised learning, has been an important and active topic in pattern analysis and artificial intelligence due to the mass of unlabeled data in real applications. Unlike supervised learning, clustering aims to divide data into $C$ non-overlapping clusters according to certain properties such that the data within the same cluster are more similar to each other than the ones in different clusters under the characteristics embedded in data points, where $C$ is the number of clusters. During the past few decades, a substantial amount of clustering algorithms have been proposed to annotate data and achieve impressive performances. However, numerous existing methods still leverage \emph{k}-means to obtain the clusters of data.

\emph{k}-means~\cite{hartigan1979algorithm} directly divides data points into $C$ clusters according to the Euclidean distances between data points and centroid of $C$ clusters which are unknown. Although it has been widely used, one of the most important disadvantages is that it cannot well treat nonlinearly separable data due to the fact that Euclidean distance cannot well characterize the intrinsic geometric structure of clusters. To alleviate this problem, the last decades has witnessed a fast paced development of \emph{k}-means variants. One of the most popular strategies is to map the original data into a high-dimensional reproducing Kernel Hilbert Space, i.e., feature space by using kernel trick such that the embedded features are linearly separable and then use \emph{k}-means to obtain labels of data. Inspired by this, kernel \emph{k}-means and many variants have been developed by selecting different kernel functions to improve clustering performance. Despite their different motivations and impressive experimental results in the literature, all of them need to directly estimate the centroid of clusters which heavily depends on the initialization of centroid matrix of clusters which is difficult in real applications, resulting in unstableness of algorithms.

 To avoid estimation of centroid matrix, Pei et al. ~\cite{Pei0WL20} presented $k$-sum by leveraging the relationship between spectral clustering and \emph{K}-means~\cite{KKSCN2004}. However, \emph{k}-sum is only suitable for single view clustering. Applied \emph{K}-means to multi-view clustering, many efficient methods have been proposed, but, the performance of existing multi-view \emph{k}-means still heavily depends on the initialization of centroid matrix which is difficult to select a suitable value in real applications. This reduces the flexibility of these algorithms. Finally, existing multi-view \emph{k}-means cannot well exploit the complementary information embedded in views. To handle the aforementioned limitations, from manifold learning perspective, we rewrite \emph{k}-means and propose a novel clustering algorithm which directly estimates the clusters of data without estimation of centroid matrix. To our best knowledge, the proposed model is the first work to rethink \emph{k}-means from manifold learning perspective and make it sense. To be specific, existing \emph{k}-means methods differ from the problem addressed in our paper in three ways:
\begin{itemize}
\item Our method builds a bridge between \emph{k}-means and manifold learning, which helps integrate manifold learning and clustering into a uniform framework.
\item Our method directly achieves the clusters of data without estimation of centroid matrix, while existing multi-view \emph{k}-means and variants don't.
\item We use butterworth filter to map distance between data points in the original space into a hidden space such that the distance between the points in the same neighborhood is small and approximately equal, and distance between points in different neighborhoods are very large. Thus, our method can well handle non-linearly separable data.
\item For multi-viewing clustering, existing \emph{k}-means and variants assume that label indicator matrices of all views are identical and do not consider the affect of different views for clustering. In our method, we relax this strict constraint and employ tensor Schatten $p$-norm minimization regularization on the third tensor which is composed of indicator labels of views. Thus, our method well exploits the complementary information embedded in views.
\item We develop an optimization algorithm for the proposed method and prove it always converges to the KKT stationary point mathematically.
\end{itemize}

\section{Related work}\label{Related work}

With the fast paced development of sensor technology, multi-view data have become the mainstream of object representation in real applications and help provide some complementary information which are important for clustering. Inspired by it, many multi-view clustering methods have been proposed during the last decades, two of the most representative techniques are subspace clustering and spectral clustering. Subspace clustering aims to find a low-dimensional subspace or embedding that fits each cluster of points. One of the most representative techniques is low-rank representation technique~\cite{2013Robust} which views original samples as dictionary to learn affinity matrix. Applied it to multi-view clustering, Luo \emph{et al}. proposed CSMSC~\cite{Luo2018Consistent}, which splits multi-view data into view consensus and view-specific by minimizing least square error. To well exploit complementary information, tensor singular value decomposition (t-SVD) based tensor nuclear norm has received much attention for multi-view clustering~\cite{XieTZLZQ18,ZhangFWLCH20}. Inspired by tensor nuclear norm, Gao et al. proposed an enhanced weighted tensor Schatten \emph{p}-norm minimization, which can well approximates the object rank and characterizes the complementary relationship between views~\cite{GaoZXXGT21}, and then presented novel multi-view subspace clustering methods~\cite{XiaZGSHG21,tcyb/XiaGWG22}. All of them have achieved good clustering performance in some scenarios, but all of them almost need post-processing such as spectral clustering to get discrete label matrix of data. This results in sub-optimal performance.

Spectral clustering converts data clustering into graph segmentation and mainly contains the following two steps: (1) Construct similarity matrix $\bf {S}$ which well depicts the structure of data and calculate Laplacian matrix of $\bf {S}$; (2) Perform eigen-decomposition on the Laplacian matrix to achieve spectral embedding for clustering~\cite{NgJW01,pami/BaiLZ23,KumarRD11}. Due to the good clustering accuracy on arbitrary shaped data and well defined mathematics, spectral clustering has become one of the most popular techniques for multi-view clustering~\cite{pami/XiaGWGDT23,WuLZ19}. To further reduce the computational complexity which is caused by similarity matrix construction and eigen-decomposition of large laplacian matrix, many fast spectral clustering methods have been presented, one of the most efficient strategies is anchor-based spectral clustering. To improve the quality of $S$ and well exploit complementary information between views, many spectral clustering based methods have been presented for both single view and multi-view clustering by different graph learning strategies such as self-weighted graph learning~\cite{ZhanZGW18} and low-rank tensor constrained spectral clustering (LTCPSC)~\cite{XuZXGG20}.

Although the motivations of existing clustering methods are different, many methods need to leverage \emph{k}-means to get clusters of data. \emph{k}-means achieves clusters of single view data by Euclidean distance between data points and centers. Applied \emph{k}-means to multi-view clustering, Bick and Scheffer proposed naive multi-view K-means~\cite{2004MultiBS}. To improve robustness of algorithm to noise and outlier, Cai et al.~\cite{2013MultiR} proposed robust multi-view \emph{K}-means (RMKMC) which leverages L21-norm instead of Euclidean distance to measure residual error. Han et al.~\cite{HanXNL22} proposed MVASM which takes into account the difference between views for clustering and exploits the view-consensus information embedded in views. Although the aforementioned methods can obtain impressive clustering performance, all of them directly obtains clusters of data in the original data space. This will reduce the flexibility of algorithm for high dimensional data due to the fact that the distance among data points can't appropriately characterize the relationship between them~\cite{1999When}.

To improve clustering accuracy for high-dimensional data which is ubiquitous in real applications, dimension reduction technique such as PCA is used to obtain low-dimensional representation and then \emph{k}-means is performed to get clusters~\cite{2002PatternC}. To get an optimal subspace which is better for clustering, a reasonable strategy is to simultaneously both learn subspace and realize clustering~\cite{icml/DingL07,DingHZS02}. For example, Ding et al.~\cite{icml/DingL07} proposed an adaptive dimension reduction approach that integrates linear discriminant analysis (LDA) into \emph{k}-means processing for single view clustering. Inspired by this efficient work, xu et al. ~\cite{xu2016discriminatively} extend it to multi-view clustering and proposed discriminatively embedded \emph{k}-means which integrates LDA into multi-view \emph{k}-means. To improve robustness of algorithm to outliers and noise, Xu et al.~\cite{XuHNL17} proposed a robust projected multi-view \emph{k}-means which is called re-weighted discriminatively embedded \emph{k}-means (RDEKM). RDEKM takes into account the difference between views and simultaneously learns projection matrix and clusters of data by \emph{k}-means in the low-dimensional space.


However, the aforementioned clustering methods can't well capture nonlinear structure embedded in data which is important for clustering. To solve this limitation, kernel trick is one of the most efficient methods. Inspired by it, kernel \emph{k}-means is proposed by using a kernel function, which maps data into a embedded high-dimensional space such that the embedded features are linearly separable~\cite{Sch1998Nonlinear}. However, it is difficult to select a suitable kernel function for data and single kernel can't well exploit the complementary information embedded in heterogeneous features. To alleviate this problem, multiple kernel strategy are with great potential to integrate complementary information into \emph{k}-means for improving clustering accuracy~\cite{2012OptimizedTPAMI,2020MultipleLX}.
 Liu et al.~\cite{2016MKKCMIR} proposed Multiple kernel \emph{k}-means with matrix-induced regularization which well considers the correlation among kernel functions. However, the aforementioned methods still need discretization procedure to get discrete label matrix. This results in sub-optimal clustering accuracy.

 Although the above \emph{k}-means variants can achieve impressive clustering performance, the performance of them heavily depends on initialization of centroid matrix of clusters. In real applications, it is difficult to select a suitable centroid matrix for clustering. Thus, the aforementioned clustering method are unstable and inflexible. To alleviate this problem, \emph{k}-means++ is proposed by using a reasonable strategy to initialize centroid matrix. Another an efficient method is \emph{k}-multiple means~\cite{2019K-mmeans}. It simultaneously learns centroid of $m$ sub-clusters, which is larger than the number of classes, and similarity matrix with connected components constraint. But all of them still involve the estimation of centroid matrix of clusters, which is difficult in real applications. To avoid centroid estimation, Pei et al. ~\cite{Pei0WL20} presented $k$-sum by leveraging the relationship between spectral clustering and \emph{K}-means~\cite{KKSCN2004}. However, \emph{k}-sum and \emph{k}-multiple means are suitable for single view clustering. To apply them to multi-view clustering, they can't well exploit the complementary information embedded in views. Different from existing work, our proposed model rethinks \emph{k}-means from manifold learning perspective and provides an efficient model which helps make \emph{k}-means sense. Moreover, our work avoids estimation of centroid matrix and well exploits the complementary information of views.


For convenience, we introduce the notations used throughout the paper. We use
bold calligraphy letters for third-order tensors, ${\bm{\mathcal {M}}} \in{\mathbb{R}} {^{{n_1} \times {n_2} \times {n_3}}}$, bold upper case letters for matrices, ${\bf{M}}$, bold lower case letters for vectors, ${\bf{m}}$, and lower case letters such as ${m_{ijk}}$ for the entries of ${\bm{\mathcal {M}}}$. Moreover, the $i$-th frontal slice of ${\bm{\mathcal {M}}}$ is ${\bm{\mathcal {M}}}^{(i)}$. $\overline {{\bm{\mathcal {M}}}}$ is the discrete Fourier transform (DFT) of ${\bm{\mathcal {M}}}$ along the third dimension, $\overline {{\bm{\mathcal {M}}}} = \mathrm{fft}({{\bm{\mathcal M}}},[ ],3)$. Thus, $\bm{{\mathcal M}} = \mathrm{ifft}({\overline {\bm{\mathcal M}}},[ ],3)$. The trace of matrix $\mathbf{M}$ is expressed as $tr(\mathbf{M})$. The Frobenius norm of ${\bm{\mathcal M}}$ is defined as ${\left\| {\bm{\mathcal M}}\right\|_F} = \sqrt {\sum\nolimits_{i,j,k} {{{\left| {{m_{ijk}}} \right|}^2}} }$.

\section{Rethinking for K-Means}

Suppose ${\mathbf{X}}^{(v)}=[\mathbf{x}_1^{(v)}\textrm{, }\ldots\textrm{ , }\mathbf{x}_N^{(v)}]^\textrm{T}\in \mathbb{R}^{N\times d_v}$ denote the data matrix of the $v$-th view ($v=1,2,\cdots,V$), where $N$ and $V$ are the number of samples and views, respectively, $d_v$ is the feature dimension of the $v$-view. The label matrix is denoted by $\mathbf{Y}=[\mathbf{y}_1\textrm{, }\ldots\textrm{ , }\mathbf{y}_N]^\textrm{T}\in \left\{0,1\right\}^{N\times C}$, where $\mathbf{y}_i$ is indicator vector of the $i$-th sample; $y_{ij}=\textrm{1}$ if $\mathbf{x}_i$ belongs to the $j$-th cluster; $y_{ij}=\textrm{0}$, otherwise. $C$ is the cluster number of data.

\emph{K}-means aims to partition data points $\mathbf{x}_1^{(v)}, \mathbf{x}_2^{(v)}, \cdots, \mathbf{x}_N^{(v)}$ into $C$ clusters according to the Euclidean distance between the data points and the centroid points $\mathbf{u}_1^{
 (v)}, \mathbf{u}_2^{(v)},\cdots,\mathbf{u}_C^{(v)}$ such that the data points in the same clusters are as close as possible and the distances between different clusters are as large as possible.
Objective function of multi-view $K$-means can be reformulated as~\cite{2004MultiBS}
\begin{equation}\label{Mk-means-indicator}
	\begin{aligned}
\mathop {\min }\limits_{{{\bf{M}}^{(v)}},{\bf{Y}}} \sum\limits_v {\left\| {{{\bf{X}}^{(v)}} - {{\bf{M}}^{(v)}}{{\bf{Y}}^T}} \right\|_F^2}
	\end{aligned}
\end{equation}
where ${\bf{M}}^{(v)} = [{\bf{u}}_1^{(v)},{\bf{u}}_2^{(v)},\cdots,{\bf{u}}_C^{(v)}]$ is centroid matrix of the $v$-th view, $\bf{u}_j^{(v)}$ denotes the centroid, i.e. mean vector of the $j$-th cluster of the $v$-th view.


For the model (\ref{Mk-means-indicator}), we have the following theorem.
\begin{theorem}\label{theorem-gao-1}
Denoted by ${A}_1\textrm{, }\ldots\textrm{ , }{A}_C$ the clusters of multi-view data ${\bf{X}^{(v)}}$, $v=1,\cdots, V$, ${\bf{u}}_j^{(v)}$ is the centroid of the $j$-th cluster of the $v$-th view, $N_v=|A_{v}|$ is the number of samples in $A_{v}$, then
\begin{equation}\label{gao-th1-1}
 \begin{array}{l}
\mathop {\min }\limits_{{{\bf{M}}^{(v)}},{\bf{Y}}} \sum\limits_v {\left\| {{{\bf{X}}^{(v)}} - {{\bf{M}}^{(v)}}{{\bf{Y}}^T}} \right\|_F^2}  = \\
\ \ \ \ \ \ \ \ \ \ \ \ \mathop {\min }\limits_{{A_1},{A_1},...,{A_c}} \sum\limits_v {\sum\limits_{i,l} {\left\| {{\bf{x}}_i^{(v)} - {\bf{x}}_l^{(v)}} \right\|_F^2{S_{il}}} }
\end{array}
\end{equation}
where
\begin{equation}\label{gao-th1-2}
{S_{il}} = \left\{ {\begin{array}{*{20}{l}}
{\left\langle {{{\bf{G}}^{{i}}},{{\bf{G}}^{{l}}}} \right\rangle ,\begin{array}{*{20}{c}}
{}
\end{array}{\bf{x}}_i^{(v)} \in {A_j}  \ and \ {\bf{x}}_l^{(v)} \in {A_j} }\\
{0,\begin{array}{*{20}{c}}
{}&{}
\end{array}otherwise}
\end{array}} \right.
\end{equation}
and
\begin{equation}\label{gao-th1-3}
{\bf{G}} = \left\{ \begin{array}{l}
{{\bf{Y}}} ,\begin{array}{*{20}{c}}
{}
\end{array} \ \ \ \ \ \ \ \ \ \ \ \ \ \ N_1 = N_2 = \cdots = N_c \\
{\bf{Y}}{({\bf{Y}^T}{\bf{Y}})^{ - {\textstyle{1 \over 2}}}},\begin{array}{*{20}{c}}
{}&{}
\end{array}otherwise
\end{array} \right.
\end{equation}
\end{theorem}

\textbf{\emph{Proof:}}

Since ${A}_1\textrm{, }\ldots\textrm{ , }{A}_C$ are the clusters of data, then,
\begin{equation}\label{Mk-means-var}
\begin{array}{l}
\mathop {\min }\limits_{{{\bf{M}}^{(v)}},{\bf{Y}}} \sum\limits_v {\left\| {{{\bf{X}}^{(v)}} - {{\bf{M}}^{(v)}}{{\bf{Y}}^T}} \right\|_F^2} =\\
\ \ \ \ \ \ \ \ \ \ \ \mathop {\min }\limits_{{A_1},{A_1},...,{A_c}} \sum\limits_v {\sum\limits_{{\bf{x}}_i^{(v)} \in {A_j}} {\left\| {{\bf{x}}_i^{(v)} - {\bf{u}}_j^{(v)}} \right\|_F^2} }
\end{array}
\end{equation}


Since all views are independent, thus, for each view such as the $v$-th view, we have
\begin{equation}\label{theorem-gao-p3}
\begin{array}{l}
\sum\limits_{{\bf{x}}_i^{(v)} \in {A_j}} {\left\| {{{\bf{x}}_i^{(v)}} - {\bf{u}}_j^{(v)}} \right\|_F^2}  \\
= \sum\limits_{{\bf{x}}_i^{(v)} \in {A_j}} {tr} ({{\bf{x}}_i^{(v)}}^T{{\bf{x}}_i^{(v)}} - 2{{\bf{x}}_i^{(v)}}^T{{\bf{u}}_j^{(v)}} + {{{\bf{u}}_j^{(v)}}^T}{{\bf{u}}_j^{(v)}})\\
 = \sum\limits_{{\bf{x}}_i^{(v)} \in {A_j}} {tr} ({{\bf{x}}_i^{(v)}}^T{{\bf{x}}_i^{(v)}} - {{\bf{x}}_i^{(v)}}^T{{\bf{u}}_j^{(v)}})
\end{array}
\end{equation}
and
\begin{equation}\label{theorem-gao-p4}
 \begin{array}{l}
\frac{1}{{2N_{v}}}\sum\limits_{{\bf{x}}_i^{(v)} \in {A_j}} {\sum\limits_{{\bf{x}}_l^{(v)} \in {A_j}} {\left\| {{{\bf{x}}_i^{(v)}} - {{\bf{x}}_l^{(v)}}} \right\|_F^2} }  \\
= \frac{1}{{2N_{v}}}\sum\limits_{{\bf{x}}_i^{(v)} \in {A_j}} {\sum\limits_{{\bf{x}}_l^{(v)} \in {A_j}} {tr} } (2{{\bf{x}}_i^{(v)}}^T{{\bf{x}}_i^{(v)}} - 2{{\bf{x}}_i^{(v)}}^T{{\bf{x}}_l^{(v)}})\\
 = \frac{1}{{2N_{v}}}\sum\limits_{{\bf{x}}_i^{(v)} \in {A_j}} {tr} (2n{{\bf{x}}_i^{(v)}}^T{{\bf{x}}_i^{(v)}} - 2n{{\bf{x}}_i^{(v)}}^T{{\bf{u}}_j^{(v)}})\\
 = \sum\limits_{{\bf{x}}_i^{(v)} \in {A_j}} {tr} ({{\bf{x}}_i^{(v)}}^T{{\bf{x}}_i^{(v)}} - {{\bf{x}}_i^{(v)}}^T{\bf{u}}_j^{(v)})
\end{array}
\end{equation}




According to Eq. (\ref{theorem-gao-p3}) and Eq. (\ref{theorem-gao-p4}), we have
\begin{equation}\label{7}
\begin{array}{l}
\sum\limits_{{\bf{x}}_i^{(v)} \in {A_j}} {\left\| {{{\bf{x}}_i^{(v)}} - {\bf{u}}_j^{(v)}} \right\|_F^2}  \\
\ \ \ \ \ \ \ \ \ \ \ \ \ = \frac{1}{{2N_{v}}}\sum\limits_{{{\bf{x}}_i^{(v)} \in {A_j}},{{\bf{x}}_l^{(v)} \in {A_j}}} {\left\| {{{\bf{x}}_i}^{(v)} - {\bf{x}}_l^{(v)}} \right\|_F^2}\\
\ \ \ \ \ \ \ \ \ \ \ \ \ = \sum\limits_{{{\bf{x}}_i^{(v)}},{{\bf{x}}_l^{(v)} }} {\left\| {{{\bf{x}}_i}^{(v)} - {\bf{x}}_l^{(v)}} \right\|_F^2{S_{il}}}
\end{array}
\end{equation}

Since all views are independent, thus we have
\begin{equation}\label{8}
\sum\limits_v {\left\| {{{\bf{X}}^{(v)}} - {{\bf{M}}^{(v)}}{{\bf{Y}}^T}} \right\|_F^2}
 = \sum\limits_v {\sum\limits_{i,l} {\left\| {{{\bf{x}}_i}^{(v)} - {\bf{x}}_l^{(v)}} \right\|_F^2{S_{il}}}}
\end{equation}

Then, the model (\ref{gao-th1-1}) holds.
\ \ \ \ \ \ \ \ \ \ \ \ \ \ \ \ \ \ \ \ \ \ \ \ \ \ \ \ \ \ \ \ \ \ \ \ \ \ \ \QEDclosed

Theorem \ref{theorem-gao-1} indicates that \emph{k}-means can be viewed as a manifold learning technique.

\begin{theorem}\label{theorem-gao-2}
Denoted by ${{\bf{D}}_v}$ distance matrix between data points of the $v$-th view, the elements ${{\bf{D}}_v}(i,l)$ of ${\bf{D}}_v$ are ${{\bf{D}}_v}(i,l) = \left\| {{\bf{x}}_i^{(v)} - {\bf{x}}_l^{(v)}} \right\|_F^2$, then
\begin{equation}\label{theorem2-gao}
\begin{array}{l}
\mathop {\min }\limits_{{{\bf{M}}^{(v)}},{\bf{Y}}} \sum\limits_v {\left\| {{{\bf{X}}^{(v)}} - {{\bf{M}}^{(v)}}{{\bf{Y}}^T}} \right\|_F^2}
 = \mathop {\min }\limits_{\bf{G} } \sum\limits_v {tr({{\bf{G}}^T}{{\bf{D}}_v}{\bf{G}})}
\end{array}
\end{equation}
\end{theorem}

\begin{proof}
For each view such as the $v$-th view, we have
\begin{equation}\label{theorem2-gao-pr1}
\begin{array}{l}
\sum\limits_{{{\bf{x}}_i^{(v)}},{{\bf{x}}_l^{(v)} }} {\left\| {{{\bf{x}}_i}^{(v)} - {\bf{x}}_l^{(v)}} \right\|_F^2{S_{il}}} \\
\ \ \ \ \ \ \ \ \ \ \ \ \ \ = \sum\limits_i {\sum\limits_l {tr} } (d_{il}^{(v)}\left\langle {{{\bf{G}}^i},{{\bf{G}}^l}} \right\rangle )\\
\ \ \ \ \ \ \ \ \ \ \ \ \ \  = \sum\limits_i {tr((d_{il}^{(v)}{{\bf{G}}^T}){{\bf{G}}^i})} \\
 \ \ \ \ \ \ \ \ \ \ \ \ \ \ = tr({{\bf{G}}^T}{{\bf{D}}_v}{\bf{G}})
\end{array}
\end{equation}

Since all views are independent, and according to Eq. (\ref{theorem2-gao-pr1}), we have
\begin{equation}\label{theorem2-gao-pr2}
 \begin{array}{l}
\sum\limits_v {\sum\limits_{{{\bf{x}}_i^{(v)}},{{\bf{x}}_l^{(v)}}} {\left\| {{{\bf{x}}_i}^{(v)} - {\bf{x}}_l^{(v)}} \right\|_F^2{S_{il}}} } \\
 \ \ \ \ \ \ \ \ = tr(\sum\limits_v {{{\bf{G}}^T}{{{\bf{D}}_v}}{\bf{G}}}) \\
 \ \ \ \ \ \ \ \ = \sum\limits_v {tr({{\bf{G}}^T}{{\bf{D}}_V}{\bf{G}})}
\end{array}
\end{equation}
then,
\begin{equation}\label{theorem2-gao-pr3}
 \begin{array}{l}
\mathop {\min }\limits_{{A_1},{A_1},...,{A_c}} \sum\limits_v {\sum\limits_{i,l} {\left\| {{\bf{x}}_i^{(v)} - {\bf{x}}_l^{(v)}} \right\|_F^2{S_{il}}} } = \\
\ \ \ \ \ \ \ \ \ \ \ \ \ \ \ \ \ \ \ \ \ \ \ \ \ \ \ \ \ \ \ \ \ \ \ \ \mathop {\min }\limits_{\bf{G} } \sum\limits_v {tr({{\bf{G}}^T}{{\bf{D}}_v}{\bf{G}})}
\end{array}
\end{equation}

According to theorem {\ref{theorem-gao-1}} and Eq. (\ref{theorem2-gao-pr3}), we have

\begin{equation}\label{theorem2-gao-pr4}
\begin{array}{l}
\mathop {\min }\limits_{{{\bf{M}}^{(v)}},{\bf{Y}}} \sum\limits_v {\left\| {{{\bf{X}}^{(v)}} - {{\bf{M}}^{(v)}}{{\bf{Y}}^T}} \right\|_F^2} \\
 \ \ \ \ \ \ \ \ \ = \mathop {\min }\limits_{{A_1},{A_1},...,{A_c}} \sum\limits_v {\sum\limits_{i,l} {\left\| {{\bf{x}}_i^{(v)} - {\bf{x}}_l^{(v)}} \right\|_F^2{S_{il}}} }\\
 \ \ \ \ \ \ \ \ \ = \mathop {\min }\limits_{\bf{G} } \sum\limits_v {tr({{\bf{G}}^T}{{\bf{D}}_v}{\bf{G}})}
\end{array}
\end{equation}
 \ \ \ \ \ \ \ \ \ \ \ \ \ \ \ \ \ \ \ \ \ \ \ \ \ \ \ \ \ \ \ \ \ \ \ \ \ \ \ \ \ \ \ \ \ \ \ \ \ \ \ \ \ \ \ \ \ \ \ \ \ \ \ \ \ \ \ \ \ \ \ \ \ \ \ \ \ \ \ \ \QEDclosed
\end{proof}

\begin{remark}\label{remark1} When $V=1$, Theorem \ref{theorem-gao-2} becomes
\begin{equation}\label{theorem3-gao}
\begin{array}{l}
\mathop {\min }\limits_{{{\bf{M}}},{\bf{Y}}} {\left\| {{{\bf{X}}} - {{\bf{M}}}{{\bf{Y}}^T}} \right\|_F^2}=\\
\mathop {\min }\limits_{{A_1},{A_1},...,{A_c}} {\sum\limits_{i,l} {\left\| {{\bf{x}}_i - {\bf{x}}_l} \right\|_F^2{S_{il}}} }= \mathop {\min }\limits_{\bf{G} } {tr({{\bf{G}}^T}{{\bf{D}}}{\bf{G}})}
\end{array}
\end{equation}
which is standard \emph{k}-means. It indicates that $k$-sum is a special case of our model (\ref{theorem2-gao}).
\end{remark}

\begin{remark}\label{remark1} When $V=1$, Theorem \ref{theorem-gao-2} becomes
Ratio-cut~\cite{HagenK92}, which is one of the most representative spectral clustering models. Ratio-cut aims to get clusters ${A}_1\textrm{, }\ldots\textrm{ , }{A}_c$ by
\begin{equation}\label{Rcut}
	\begin{aligned}
		\mathop {\min }\limits_{{A_1},{A_1},...,{A_c}} {\sum\limits_{l=1}^{c}{\frac{\textrm{cut}\left(\mathbf{A}_l, \overline{\mathbf{A}}_l\right)}{\left|\mathbf{A}_l\right|}}}
	\end{aligned}
\end{equation}
where $\overline{\mathbf{A}}_l$ is the complement of $\mathbf{A}_l$; $\left| \mathbf{A}_l \right|$ denotes the number of samples in $\mathbf{A}_l$, and $\textrm{cut}\left(\mathbf{A}_l, \overline{\mathbf{A}}_l\right)=\sum_{i\in\mathbf{A}_l\textrm{, }j\in\overline{\mathbf{A}}_l}{w_{ij}}$, $w_{ij}$ is the weight between the $i$-th node and $j$-th node.

By simple algebra, the problem (\ref{Rcut}) becomes
\begin{equation}\label{Rcut2}
\quad \mathop {\min }\limits_{\mathbf{Y}\in \bm{\mathcal{I}}} {\mathrm{\emph{tr}}( (\mathbf{Y}^T\mathbf{Y})^{-1/2} \mathbf{Y}^T\mathbf{L}\mathbf{Y} (\mathbf{Y}^T\mathbf{Y})^{-1/2})}
\end{equation}
where $\mathbf{L}$ is Laplacian matrix.

According to the definition of $\bf{G}$ (See Eq. (\ref{gao-th1-3})), we have that the model (\ref{Rcut2}) can be rewritten as
\begin{equation}\label{Rcut2-Km}
\quad \mathop {\min }\limits_{\mathbf{G}} {\mathrm{\emph{tr}}( \mathbf{G}^T\mathbf{L}\mathbf{G} )}
\end{equation}
When ${\bf{L}}$ becomes ${\bf{D}}$, the model (\ref{Rcut2-Km}) becomes the model (\ref{theorem3-gao}). It means that Spectral clustering is also a special case of our model (\ref{theorem2-gao}).
\end{remark}

Combining the aforementioned insight analysis, we have that, from manifold learning perspective, we make \emph{k}-means sense and integrate \emph{k}-means, manifold learning and spectral clustering into a uniform framework.

%

\section{Methodogy}
\subsection{Motivation and Objective}
In reality, different views contain different characteristics of the object, thus, different views should have different similarity matrices, resulting to different label matrices of different views. However, the model (\ref{theorem2-gao}) neglects this, resulting in sub-optimal performance. To further improve the performance, combining the aforementioned insight analysis, we
rewrite the model (\ref{theorem2-gao}) as the following unified framework for multi-view clustering:
\begin{equation}\label{model0}
\begin{aligned}
&\mathop {{\rm{min}}}\limits_{{{\bf{Y}}^{(v)}},{\alpha _{\rm{v}}}} \sum\limits_{v{\rm{ = 1}}}^{{V}} {\alpha _{\rm{v}}^{{r}}} {\rm{tr}}({{\bf{G}}^{(v)T}}{{\bf{D}}^{(v)}}{{\bf{G}}^{(v)}}){\rm{ + }}\lambda \sum\limits_{v{\rm{ = 1}}}^V {\cal R} ({{\bf{Y}}^{(v)}})\\
		&\emph{\textrm{s.t.}} \quad \quad {\textbf{Y}^{\textrm{(\emph{v})}}\in \textrm{Ind}}\textrm{, }\sum\limits_{\textrm{\emph{v}=1}}^\textrm{\emph{V}} {{\alpha_\textrm{\emph{v}}}\textrm{ = 1}} ,{\rm{ }}{\alpha_\textrm{\emph{v}}} \ge {{\textrm{0}}}
	\end{aligned}
\end{equation}
where $\textbf{Y}^{\textrm{(\emph{v})}}$ is the label matrix of the $v$-th view; $\mathcal{R}\textrm{(}\bullet\textrm{)}$ represents the regularizer on ${\mathbf{Y}^{(v)}}$; $\alpha_v^r$ is the adaptive weight for the $v$-th view, which characterizes the importance of the $v$-th view, $V$ is the number of views, $\lambda$ is a trade-off parameter, and
\begin{equation}\label{gao-th1-33}
{\bf{G}}^{(v)} = \left\{ \begin{array}{l}
{{\bf{Y}}}^{(v)} ,\begin{array}{*{20}{c}}
{}
\end{array} \ \ \ \ \ \ \ \ \ \ \ \ \ \ N_1 = N_2 = \cdots = N_c \\
{\bf{Y}}^{(v)}{({{\bf{Y}}^{(v)T}}{\bf{Y}}^{(v)})^{ - {\textstyle{1 \over 2}}}},\begin{array}{*{20}{c}}
{}&{}
\end{array}otherwise
\end{array} \right.
\end{equation}

\begin{figure}[!t]
	\centering
	\includegraphics[width=1.0\linewidth]{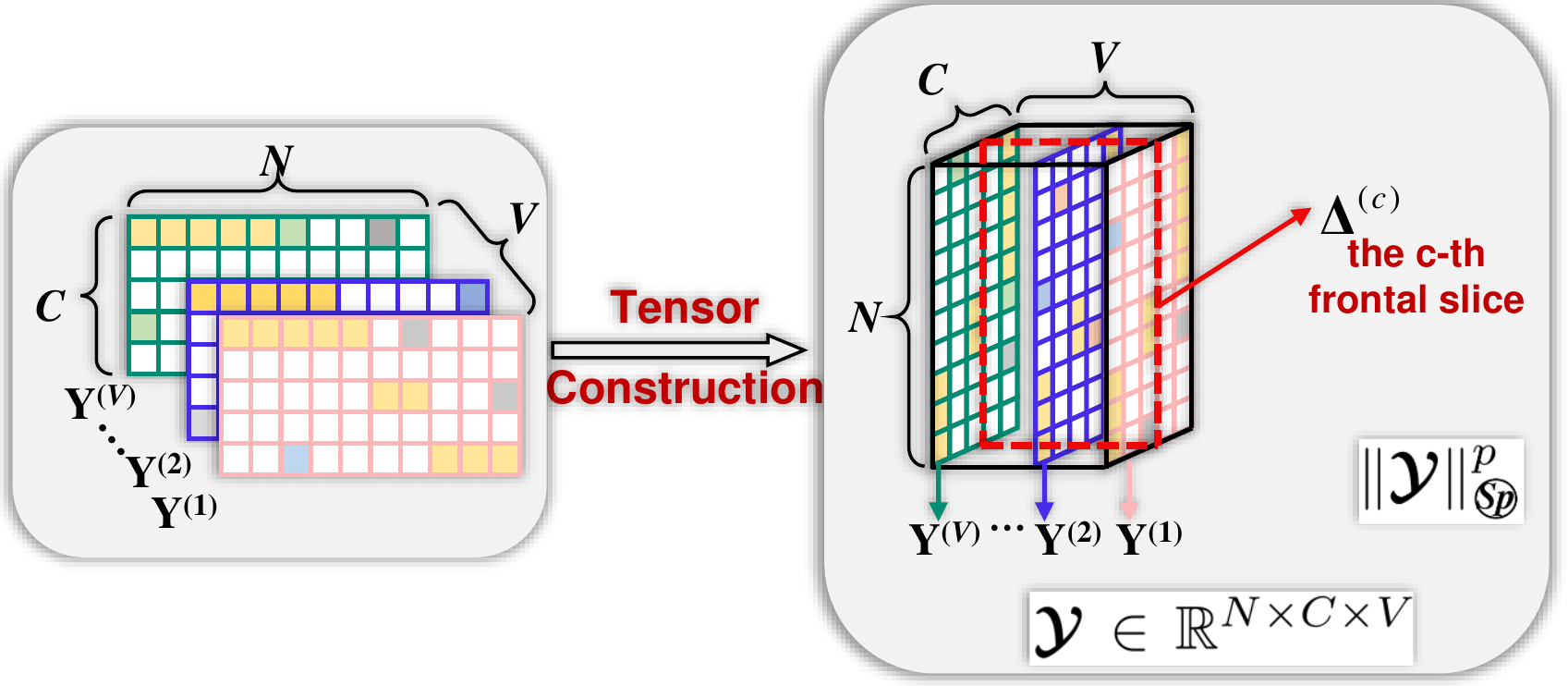}
	\caption{Construction of tensor $\bm{\mathcal{Y}} \in \mathbb{R}^{N\times C\times V}$. $\bm{\Delta}^{(c)}$ denotes the $c$-th frontal slice of $\bm{\mathcal{Y}}$ ($c\in\{1\textrm{, }2\textrm{, }\cdots\textrm{, }C\}$).}
	\label{fig1}
	\vspace{-3mm}
\end{figure}

In the model (\ref{model0}), the second term aims to minimize the divergence between ${\mathbf{Y}^{(v)}}$. To solve this problem, a naive method is to leverage squared \emph{F}-norm to learn the discrete label matrix. As we all know, $F$-norm is a one-dimensional and pixel-wise measurement method. Thus, it cannot well exploit the complementary information embedded in ${\mathbf{Y}^{(v)}}$. It is evident that the low rank approximation using tensor Schatten $p$-norm performs very well in exploiting the complementary information embedded in views~\cite{GaoZXXGT21,XiaZGSHG21}, thus, we minimize the divergence between ${\bf{Y}^{v}}$ by using the tensor Schatten $p$-norm minimization on the third-order tensor ${\mathcal Y}$ which consists of ${\bf{Y}^{v}}$. Thus, we have

\begin{equation}\label{model1}
	\begin{aligned}
		&\mathop {{\rm{min}}}\limits_{{{\bf{Y}}^{(v)}},{\alpha _{\rm{v}}}} \sum\limits_{v{\rm{ = 1}}}^{{V}} {\alpha _{\rm{v}}^{{r}}} {\rm{tr}}({{\bf{G}}^{(v)T}}{{\bf{D}}^{(v)}}{{\bf{G}}^{(v)}})\textrm{+}\lambda{\| {\bm{\mathcal Y}} \|_{\Sp}^p}\\
		&\emph{\textrm{s.t.}} \quad \quad {\textbf{Y}^{\textrm{(\emph{v})}}\in \textrm{Ind}}\textrm{, }\sum\limits_{\textrm{\emph{v}=1}}^\textrm{\emph{V}} {{\alpha_\textrm{\emph{v}}}\textrm{ = 1}} ,{\rm{ }}{\alpha_\textrm{\emph{v}}} \ge {{\textrm{0}}}
	\end{aligned}
\end{equation}
where the $i$-th lateral slice of $\bm{\mathcal{Y}} \in \mathbb{R}^{N\times V\times C}$ is ${\textbf{Y}}^{(v)}$ (See Fig. \ref{fig1}); $\left\|\bullet \right\|_{\Sp}$ is the tensor Schatten $p$-norm (see Definition~\ref{Sp-norm}).

\begin{definition}\label{Sp-norm}(Tensor Schatten p-norm~\cite{GaoZXXGT21})
Given ${\bm{\mathcal Y}}\in{\mathbb{R}}^{n_1 \times n_2 \times n_3}$, $h=\min(n_1,n_2)$, the tensor Schatten $p$-norm of tensor ${\bm{\mathcal Y}}$ is defined as
\begin{equation}
\begin{array}{c}
{\left\| {\bm{\mathcal Y}} \right\|_{{\Sp}}} = {\left( {\sum\limits_{i = 1}^{{n_3}} {\left\| {{{\overline {\bm{\mathcal Y}} }^{(i)}}} \right\|}_{{\Sp}}^p} \right)^{\frac{1}{p}}} = {\left( {\sum\limits_{i = 1}^{{n_3}} {\sum\limits_{j = 1}^h {{\sigma _j}{{\left( {{{\overline {\bm{\mathcal Y}} }^{(i)}}} \right)}^p}} } } \right)^{\frac{1}{p}}}
\end{array}\label{4}
\end{equation}
where ${\sigma _j}(\overline{\bm{\mathcal Y}}^{(i)})$ denotes the j-th singular value of $\overline{\bm{\mathcal Y}}^{(i)}$.
		\label{definition1}
	\end{definition}

\begin{remark}\label{remark1} when $p=1$, tensor Schatten p-norm of ${\bm{\mathcal Y}}\in{\mathbb{R}}^{n_1 \times n_2 \times n_3}$ becomes tensor nuclear norm~\cite{GaoZXXGT21}, i.e., ${\left\| {\bm{\mathcal Y}} \right\|_{*}} ={ {\sum\limits_{i = 1}^{{n_3}} {\sum\limits_{j = 1}^h {{\sigma _j}{{\left( {{{\overline {\bm{\mathcal Y}} }^{(i)}}} \right)}}} } } }$. Take matrix Schatten p-norm as an example, that is for
		${\bm{Y}} \in\mathbb{R}^{n_1\times n_2}$ and the singular values of ${\bm{Y}}$ denoted by $\sigma_1,\ldots,\sigma_{h}$, we have
		$\|{\bm{ Y}}\|^p_{\Sp}=\sigma_1^p+\cdots+\sigma_{h}^p,\ p>0.$ Nie et al.~\cite{nie2012robust} has shown that $\lim_{p\to  0}\|{\bm{ Y}}\|^p_{\Sp}=\#\{i: \sigma_i\ne 0\}=\hbox{rank}({\bm{Y}}).$ And  for $0 \le p \le 1$, i.e. when $p$ is appropriately chosen, the Schatten p-norm can give us quite effective improvements for a tighter approximation of the rank function.
\end{remark}


\begin{remark}
	The regularizer in the proposed objective (\ref{model1}) is used to explore the complementary information embedded in inter-views cluster assignment matrices $\textbf{Y}^{(v)}$ ($v=1,2,\cdots,V$). Fig.~\ref{fig1} shows the construction of tensor $\bm{\mathcal{Y}}$, it can be seen that the $c$-th frontal slice $\bm{\Delta}^{(c)}$ describes the similarity between $N$ sample points and the $c$-th cluster in different views. The idea cluster assignment matrix $\textbf{Y}^{(v)}$ should satisfy that the relationship between $N$ data points and the $c$-th cluster is consistent in different views. Since different views usually show different cluster structures, we impose tensor Schatten  p-norm minimization~\cite{GaoZXXGT21} constraint on $\bm{\mathcal{Y}}$, which can make sure each $\bm{\Delta}^{(c)}$ has spatial low-rank structure. Thus $\bm{\Delta}^{(c)}$ can well characterize the complementary information embedded in inter-views.
\end{remark}

\subsection{Optimization}

In (\ref{model1}), the main difficulty in finding the discrete assignment matrix is due to the non-convex property. To this end, we suppose the number of samples in each cluster are equal. In this case, $N_1=N_2=\cdots=N_c$, according to (\ref{gao-th1-33}), we have ${\mathbf{Y}^{(v)}} = {\mathbf{G}^{(v)}}$.
 Thus, (\ref{model1}) becomes
\begin{equation}\label{objfinal}
	\begin{aligned}
		&\mathop {\textrm{min}}\limits_{\textbf{Y}^{\textrm{(\emph{v})}}\textrm{, }\alpha_\emph{\textrm{v}}}\sum\limits_{\textrm{\emph{v}=1}}^{\emph{\textrm{V}}} \alpha_\emph{\textrm{v}}^\emph{\textrm{r}}\mathrm{tr}\textrm{(}{{\textbf{Y}^{\textrm{(\emph{v})}}}^\textrm{T}}\textbf{D}^{\textrm{(\emph{v})}} \textbf{Y}^{\textrm{(\emph{v})}}\textrm{)}\textrm{+}\lambda{\| {\bm{\mathcal Y}} \|_{\Sp}^p}\\
		&\emph{\textrm{s.t.}} \quad \quad {\textbf{Y}^{\textrm{(\emph{v})}}\in \textrm{Ind}}\textrm{, }\sum\limits_{\textrm{\emph{v}=1}}^\textrm{\emph{V}} {{\alpha_\textrm{\emph{v}}}\textrm{ = 1}} ,{\rm{ }}{\alpha_\textrm{\emph{v}}} \ge {{\textrm{0}}}
	\end{aligned}
\end{equation}

Inspire by the augmented Lagrange multiplier (ALM)~\cite{LinLS11}, we introduce an auxiliary variable $\bm{\mathcal{J}}$ and rewrite (\ref{objfinal}) as
\begin{equation}\label{obj function2}
	\begin{aligned}
		\bm{\mathcal{L}}(\bm{\mathcal{Y}},\bm{\mathcal{J}} ) &= \sum\limits_{\textrm{\emph{v}=1}}^{\emph{\textrm{V}}} \alpha_\emph{\textrm{v}}^\emph{\textrm{r}}\mathrm{tr}\textrm{(}{{\textbf{Y}^{\textrm{(\emph{v})}}}^\textrm{T}}\textbf{D}^{\textrm{(\emph{v})}} \textbf{Y}^{\textrm{(\emph{v})}}\textrm{)}\textrm{+}\lambda{\| {\bm{\mathcal J}} \|_{\Sp}^p}\\
		&\quad + {\langle{\bm{\mathcal{Q}}, \bm{\mathcal{Y}}-\bm{\mathcal{J}}}\rangle} + \frac{\mu}{2}\|\bm{\mathcal{Y}}-\bm{\mathcal{J}}\|_F^2
	\end{aligned}
\end{equation}
where $\bm{\mathcal{Q}}$ is Lagrange multipliers; $\mu$ is a penalty parameter. The optimization process could be separated into three steps:

$\bullet$ \textbf{Solving $\mathbf{Y}^{(v)}$ with fixed $\alpha_v$ and  $\bm{\mathcal{J}}$.} When $\alpha_v$ and $\bm{\mathcal{J}}$ are fixed, the optimization w.r.t. $\mathbf{Y}^{(v)}$ in  (\ref{obj function2}) becomes
\begin{equation}\label{SolvingY}
	\begin{aligned}
		\mathop {\textrm{min} }\limits_{\mathbf{Y}^{(v)}\in \textrm{Ind}}& \sum\limits_{\textrm{\emph{v}=1}}^{\emph{\textrm{V}}} \alpha_\emph{\textrm{v}}^\emph{\textrm{r}}\mathrm{tr}\textrm{(}{{\textbf{Y}^{\textrm{(\emph{v})}}}^\textrm{T}}\textbf{D}^{\textrm{(\emph{v})}} \textbf{Y}^{\textrm{(\emph{v})}}\textrm{)}\textrm{+}\frac{\mu}{\textrm{2}}\|\bm{\mathcal{Y}}\textrm{-}\bm{\mathcal{J}}\textrm{+}\frac{\bm{\mathcal{Q}}}{\mu}\|_F^2\\
	\end{aligned}
\end{equation}

Since all $\mathbf{Y}^{(v)}$'s ($v=1,\cdots,V$) are independent, then model (\ref{SolvingY}) can be decomposed into $V$ independent sub-optimization problems. In other words, we can obtain $\mathbf{Y}^{(v)}$ ($v=1,\cdots,V$) of the corresponding view by solving
\begin{equation}\label{SolvingYv}
	\begin{aligned}
		\mathop {\textrm{min} }\limits_{\mathbf{Y}^{(v)}\in \textrm{Ind}}& 		 \alpha_\emph{\textrm{v}}^\emph{\textrm{r}}\mathrm{tr}\textrm{(}{{\textbf{Y}^{\textrm{(\emph{v})}}}^\textrm{T}}\textbf{D}^{\textrm{(\emph{v})}} \textbf{Y}^{\textrm{(\emph{v})}}\textrm{)}\textrm{+}\frac{\mu}{\textrm{2}}\|\mathbf{Y}^{(v)}\textrm{-}\mathbf{S}^{(v)}\|_F^2
	\end{aligned}
\end{equation}
where $\mathbf{S}^{(v)}\textrm{=}\mathbf{J}^{(v)}\textrm{-}\frac{1}{\mu}\mathbf{Q}^{(v)}$.

It is hard to directly get the optimal solution of (\ref{SolvingYv}) due to the discrete constraint. Sine all rows of $\mathbf{Y}^{(v)}$ are independent, we sequentially solve $\mathbf{Y}^{(v)}$ row by row with fixed the other rows. To update the $k$-th row, we assume the other rows of $\mathbf{Y}^{(v)}$ are known. Then, the optimization w.r.t. the $k$-th row in (\ref{SolvingYv}) becomes
\begin{equation}\label{SolveYFinal}
	\begin{aligned}
		& \mathop {\textrm{min} }\limits_{\mathbf{y}_{k}^{\textrm{(\emph{v})}}\in \textrm{Ind}}   \alpha_\emph{\textrm{v}}^\emph{\textrm{r}} \sum\limits_{{{i,j}=1}}^{\emph{\textrm{N}}} d_{ij}^{\textrm{(\emph{v})}} \mathrm{tr}(\mathbf{y}_{i}^{\textrm{(\emph{v})}}{\mathbf{y}_{j}^{\textrm{(\emph{v})}}}^{\textrm{T}})+\frac{\mu}{2}\|\mathbf{y}_k^{\textrm{(\emph{v})}}\textrm{-}\mathbf{s}_k^{\textrm{(\emph{v})}}\|_F^2 \\
		 \Leftrightarrow & \mathop {\textrm{min} }\limits_{\mathbf{y}_{k}^{\textrm{(\emph{v})}}\in \textrm{Ind}}   \alpha_\emph{\textrm{v}}^\emph{\textrm{r}} {\mathbf{y}_{k}^{\textrm{(\emph{v})}}}^{\textrm{T}} (2\sum\limits_{
		 	 i\mathrm{=}1,i\mathrm{\neq}k}^{\emph{\textrm{N}}} d_{ki}^{\textrm{(\emph{v})}} \mathbf{y}_{i}^{\textrm{(\emph{v})}})-\mu{\mathbf{y}_k^{\textrm{(\emph{v})}}}^\textrm{T}\mathbf{s}_k^{\textrm{(\emph{v})}}
\end{aligned}
\end{equation}

Since $d_{kk}^{\textrm{(\emph{v})}}\mathrm{=}0$, thus, the problem~(\ref{SolveYFinal}) can be rewritten as
\begin{equation}\label{Solveyi1}
	\begin{aligned}
	& \mathop {\textrm{min} }\limits_{\mathbf{y}_{k}^{\textrm{(\emph{v})}}\in \textrm{Ind}}   2\alpha_\emph{\textrm{v}}^\emph{\textrm{r}} {\mathbf{y}_{k}^{\textrm{(\emph{v})}}}^{\textrm{T}} (\sum\limits_{i\mathrm{=}1}^{\emph{\textrm{N}}} d_{ki}^{\textrm{(\emph{v})}} \mathbf{y}_{i}^{\textrm{(\emph{v})}})-\mu{\mathbf{y}_k^{\textrm{(\emph{v})}}}^\textrm{T}\mathbf{s}_k^{\textrm{(\emph{v})}} \\
	\Leftrightarrow & \mathop {\textrm{min} }\limits_{\mathbf{y}_{k}^{\textrm{(\emph{v})}}\in \textrm{Ind}}  {\mathbf{y}_{k}^{\textrm{(\emph{v})}}}^{\textrm{T}} (2\alpha_\emph{\textrm{v}}^\emph{\textrm{r}} {\textbf{Y}_*^{\textrm{(\emph{v})}}}^\textrm{T} \mathbf{d}_k^{\textrm{(\emph{v})}} -\mu\mathbf{s}_k^{\textrm{(\emph{v})}} )
\end{aligned}
\end{equation}
where  $\mathbf{d}_k^{\textrm{(\emph{v})}}\textrm{=}[d_{k1}^{\textrm{(\emph{v})}}, d_{k2}^{\textrm{(\emph{v})}},\cdots,d_{kn}^{\textrm{(\emph{v})}} ]^\textrm{T}$, $d_{ii}^{\textrm{(\emph{v})}}\textrm{=}0$;  $\textbf{Y}_*^{\textrm{(\emph{v})}}$ is the cluster assignment matrix before $\mathbf{y}_k^{(v)}$ is updated. Then, the optimal solution of  (\ref{SolvingY}) can be reformulated as

\begin{equation}\label{Solveyi}
	\begin{aligned}
		y_{ip}^{\textrm{(\emph{v})}}=\begin{cases}
			1\textrm{,} & p= \mathop{\textrm{arg min} }\limits_{j}(2\alpha_\emph{\textrm{v}}^\emph{\textrm{r}} {\textbf{Y}_*^{\textrm{(\emph{v})}}}^\textrm{T} \mathbf{d}_i^{\textrm{(\emph{v})}} -\mu\mathbf{s}_i^{\textrm{(\emph{v})}})_j\\
			0\textrm{,} & \quad\quad\mbox{\textrm{otherwise}} .
		\end{cases}
	\end{aligned}
\end{equation}

$\bullet$ \textbf{Solving $\alpha_v$ with fixed $\bm{\mathcal{Y}}$ and  $\bm{\mathcal{J}}$.} When $\bm{\mathcal{Y}}$ and $\bm{\mathcal{J}}$ are fixed, the optimization w.r.t. $\alpha_v$ in  (\ref{obj function2}) is equivalent to
\begin{equation}\label{alpha}
	\begin{aligned}
		&\mathop {\textrm{min}}\limits_{\alpha_\emph{\textrm{v}}}\sum\limits_{\textrm{\emph{v}=1}}^{\emph{\textrm{V}}} \alpha_\emph{\textrm{v}}^\emph{\textrm{r}}\mathrm{tr}\textrm{(}{{\textbf{Y}^{\textrm{(\emph{v})}}}^\textrm{T}}\textbf{D}^{\textrm{(\emph{v})}} \textbf{Y}^{\textrm{(\emph{v})}}\textrm{)} &\emph{\textrm{s.t.}} \quad \sum\limits_{\textrm{\emph{v}=1}}^\textrm{\emph{V}} {{\alpha_\textrm{\emph{v}}}\textrm{ = 1}} ,{\rm{ }}{\alpha_\textrm{\emph{v}}} \ge {{\textrm{0}}}
	\end{aligned}
\end{equation}

According to the Lagrange multiplier method, we establish the Lagrangian function as
\begin{equation}\label{alpha_LFun}
	\begin{aligned}
		\bm{\mathcal{L}}(\alpha_1,\ldots,\alpha_V,\gamma) = \sum\limits_{v=1}^{V}\alpha_v^r\mathbf{M}_v + \gamma\left(1-\sum\limits_{v=1}^{V}{\alpha_v}\right)
	\end{aligned}
\end{equation}
where $\gamma$ is a parameter; ${\mathbf{M}_v} =
\mathrm{tr}\textrm{(}{{\textbf{Y}^{\textrm{(\emph{v})}}}^\textrm{T}}\textbf{D}^{\textrm{(\emph{v})}} \textbf{Y}^{\textrm{(\emph{v})}}\textrm{)}$.

According to the KKT conditions, we have
\begin{equation}\label{alpha_kkt}
	\begin{aligned}
		\begin{cases}
			\frac{\partial \bm{\mathcal{L}}}{\partial \alpha_v} = r \alpha_v^{r\mathrm{-}1} \textbf{M}_v \mathrm{-} \gamma \mathrm{-} \theta_v = 0 ,\quad  v=1,\ldots,V \\
			\theta_v \alpha_v =0 ,\quad  v=1,\ldots,V \\
			\theta_v \geq 0 ,\quad  v=1,\ldots,V \\
			1\mathrm{-}\sum\limits_{v=1}^{V}{\alpha_v} = 0
		\end{cases}
	\end{aligned}
\end{equation}

 In light of some simple algebra, we have
\begin{equation}\label{alpha_opt}
	\begin{aligned}
		\alpha_v=\frac{(\mathbf{M}_v)^{\frac{1}{1-r}}} {\sum\limits_{v=1}^{V}(\mathbf{M}_v)^{\frac{1}{1-r}}}
	\end{aligned}
\end{equation}

$\bullet$ \textbf{Solving $\bm{\mathcal{J}}$ with fixed $\bm{\mathcal{Y}}$ and $\alpha_v$.} In this case, $\bm{\mathcal{J}}$ can be obtained by solving
\begin{equation}\label{SolveJ}
	\begin{aligned}
		\bm{\mathcal{J}}^*=\mathop{\arg\min}_{\bm{\mathcal{J}}} \frac{\lambda}{\mu}{\| {\bm{\mathcal J}} \|_{\Sp}^p} + \frac{1}{2}\|\bm{\mathcal{Y}}-\bm{\mathcal{J}}+\frac{\bm{\mathcal{Q}}}{\mu}\|_F^2
	\end{aligned}
\end{equation}

 To solve model (\ref{SolveJ}), we first introduce Theorem~\ref{T1}~\cite{GaoZXXGT21}.
	\begin{theorem}\label{T1}
		Given third-order tensor ${\bm{{\cal A}}{\in \mathbb{R}^{{n_1} \times {n_2} \times {n_3}}}}$ whose t-SVD denotes by ${\bm{{\cal A}}\! = \!\bm{{\cal U}}*\bm{{\cal S}}\!*\!{\bm{{\cal V}}^\mathrm{T}}}$. For the problem	
		\begin{equation}\label{19}
			\arg {\rm{ }}\mathop {\min }\limits_{\bm{{\cal X}}} \mu\left\| \bm{{\cal X}} \right\|_{\Sp}^p + \frac{1}{2}\left\| {\bm{{\cal X}} - \bm{{\cal A}}} \right\|_F^2,
		\end{equation}
the optimal solution is
		\begin{equation}\label{20}
			{{\bm{{\cal X}}}^ * } = {\rm{ }}{\Gamma _{\mu}}\left[ \bm{{\cal A}} \right]{\rm{ }} = {\rm{ }}\bm{{\cal U}} * \mathrm{ifft}({P_{\mu}}\left( {\overline {\bm{{\cal A}}} } \right)) * {\bm{{\cal V}}^\mathrm{T}},
		\end{equation}
where ${P_{\mu}}(\bar {\bm{{\cal A}} }) {\in \mathbb{R}^{{n_1} \times {n_2} \times {n_3}}}$ is a f-diagonal tensor whose diagonal elements can be obtained by the GST algorithm introduced in Lemma 1 of \cite{GaoZXXGT21}.
	\end{theorem}
	
	According to Theorem \ref{T1}, let ${\bm{{\mathcal{Y}}}{\rm{ }} + {\rm{ }}\frac{\bm{{\cal Q}}}{\mu }{\rm{ }}}=\bm{{\mathcal{U}}}*\bm{{\Sigma}}*\bm{{\mathcal{V}}}^\mathrm{T}$, then the solution of (\ref{SolveJ}) is	
\begin{equation}\label{21}
\begin{aligned}
{{\bm{{\cal J}}}^ * } &= {\rm{ }}\Gamma {{\rm{ }}_{\frac{\lambda}{\mu } }}\left[ {\bm{{\mathcal{Y}}}{\rm{ }} + {\rm{ }}\frac{\bm{{\cal Q}}}{\mu }{\rm{ }}} \right]{\rm{ }}= {\rm{ }}{{\bm{{\cal U}}}_{\bm{{\cal }}}} * \mathrm{ifft}({P_{\frac{\lambda}{\mu } }}\left( {\overline {\bm{{\mathcal{Y}}}{\rm{ }} + {\rm{ }}\frac{\bm{{\cal Q}}}{\mu }} {\rm{ }}} \right)) * {\bm{{\cal V}}}^{\mathrm{T}}{\rm{ }}
\end{aligned}
\end{equation}

Finally, Algorithm~\ref{A1} lists the pseudo code of the optimization procedure about the model (\ref{objfinal}).

\begin{algorithm}[!t]
	\caption{Solving the Model (\ref{obj function2})}
	\label{A1}
	\LinesNumbered
	\KwIn{ Data matrices $\{{\mathbf{X}}^{(v)}\}_{v=1}^{V}\in \mathbb{R}^{N\times d_v}$; anchors number $\theta$; cluster number $C$.}
	\KwOut{ Cluster assignment matrix $\mathbf{K}$}
	\textbf{Initialize}: $\lambda$, $\Omega$, $r$, $p$, $\bm{{\mathcal J}}$, ${{\bf{D}}^{(v)}}$, $\bm{{\mathbf Y^{(v)}}}$, ${\alpha_v}={\frac{1}{V}}$, ($v=1\ldots V$), $\rho = 1.1$, $\mu={10^{ - 4}}$, ${\mu _{\max }} = {10^{10}}$.\\	
	Construct $\mathbf{H}^{(v)}$ according to Eq. (\ref{consD});\\
	\While{not converge}{
		Update ${\bm{{\mathcal Y}}}$ by solving Eq. (\ref{Solveyi}); \\
		Update ${\bm{{\mathcal J}}}$ by solving Eq. (\ref{SolveJ});\\
		Update ${\bm{{\mathcal Q}}}$ by ${\bm{{\mathcal Q}}}{\rm{ = }}{\bm{{\mathcal Q}}}{\rm{ + }}\mu {\rm{(}}{\bm{{\mathcal Y}}}{\rm{ - }}{\bm{{\mathcal J}}}{\rm{)}}$;\\
		Update ${\alpha_v}$ by Eq. (\ref{alpha_opt}) ($v=1\ldots V$); \\
		Update $\mu$ by $\mu  = \min \left( {\rho \mu ,{\mu _{\max }}} \right)$;\\
	}
	Calculate the cluster assignment matrix $\textbf{K}$ by $k_{il} = \begin{cases}
		1\textrm{,} &  l=\mathop{\arg\max}\limits_j \left(\sum\limits_{v=1}^{\textrm{\emph{V}}}\alpha_v^r \textbf{Y}^{(v)}\right)_{ij}  \\
		0\textrm{,} & \quad \quad\mbox{otherwise}.
	\end{cases}  $\\
	\textbf{return}: The cluster assignment matrix $\mathbf{K}$
\end{algorithm}

\subsection{$\mathbf{D}^{(v)}$ Construction}

 It is well known that $K$-means cannot well separate clusters which are non-linearly separable in input space due to the Euclidean distance between data poins. To address this issue, it is useful to take the advantage of adjacency matrix, which can well characterize the intrinsic structure of arbitrarily-shaped clusters, and good property of anchor graph, we use small adjacency matrix $\mathbf{B}^{(v)}\in \mathbb{R}^{N\times \theta}$ to construct distance matrix $\mathbf{D}^{(v)}$, where $\theta\ll N$ is the number of anchors. To be specific, we firstly construct $\mathbf{B}^{(v)}\in \mathbb{R}^{N\times \theta}$ by~\cite{pami/XiaGWGDT23}, then calculate the symmetric and doubly-stochastic adjacency matrix $\mathbf{W}^{(v)}$ by~\cite{LiuHC10}, i.e.,
\begin{equation}\label{consW}
	\begin{aligned}
		\mathbf{W}^{(v)} = \mathbf{B}^{(v)} {\Delta ^{(v)}}^{\textrm{-1}} {\mathbf{B}^{(v)}}^\textrm{T}
	\end{aligned}
\end{equation}
where $\Delta^{(v)}\in\mathbb{R}^{\theta\times\theta}$ is a diagonal matrix and $\Delta _{jj}^{(v)} =\sum\nolimits_{i=1}^N {b^{v}_{ij}}$, $b^{v}_{ij}$ is the $i$-th row $j$-th column element of ${\bf{B}}^{v}$.

According to the relationship between similarity and distance for data points, to improve the clustering performance and stableness of \emph{k}-means, inspired by the Butterworth filters, we calculate $d^{v}_{ij}$, which is the $i$-th row $j$-th column element of ${\bf{D}}^{v}$, by
\begin{equation}\label{consD}
	\begin{aligned}
		h^{(v)}_{ij} = \sqrt{\frac{1}{1+ ({\frac{w^{(v)}_{ij}}{\Omega}})^4 }}
	\end{aligned}
\end{equation}
where $\Omega$ is a hyperparameter.

Fig.\ref{bw} shows the curve of the model (\ref{consD}). It can be seen that the model (\ref{consD}) increases the distance between similar and dissimilar samples, while reducing the distance between similar samples. This indicates that Eq. (\ref{consD}) helps increase the separability of samples.


\begin{figure}[!t]
	\centering
	\includegraphics[width=0.7\linewidth]{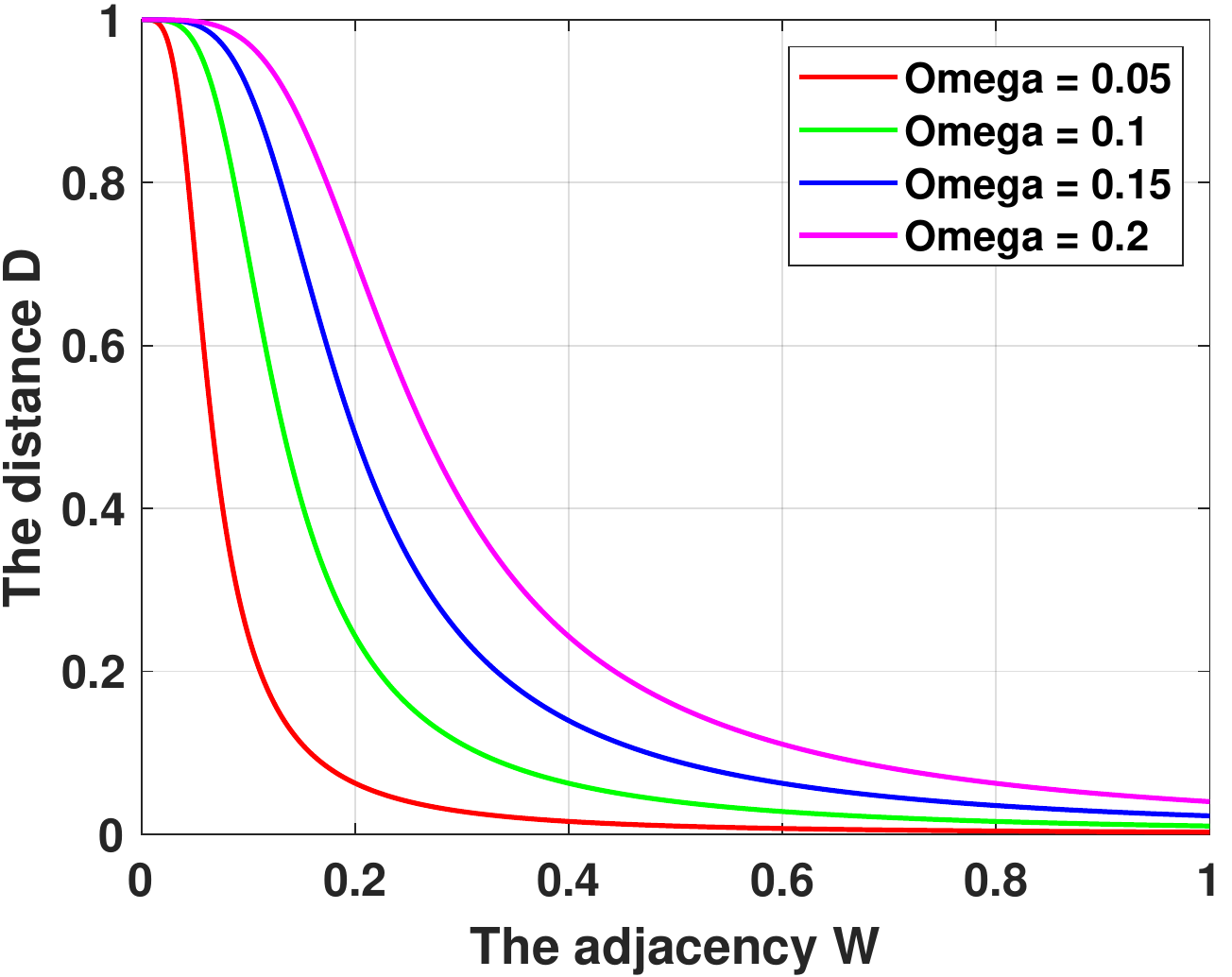}
	\caption{Distance vs. similarity under different $\Omega$'s.}
	\label{bw}
\end{figure}

\subsection{Complexity Analysis}
To better illustrate the efficiency of the proposed model, computational complexity and storage complexity are analyzed in this section.

\textbf{\emph{Computational complexity:}} Our method consists of two stages: 1) Construction of $\{\mathbf{H}^{(v)}\}_{v=1}^V$, 2) Iterative updating  (\ref{obj function2}). The first stage takes $\bm{\mathcal{O}}(VN\theta d + VN\theta\log(\theta))$ for the anchor graph construction, where $d=\sum\nolimits_{v=1}^V d_v$; $V$, $\theta$ and $N$ are the number of views, anchors and samples, respectively. The second stage mainly focuses on solving two variables ($\bm{\mathbf{Y}}^{(v)}$ and $\bm{\mathcal{J}}$), the complexity in updating these two variables are $\bm{\mathcal{O}}(NC)$ and $\bm{\mathcal{O}}(VNC\log(VN)+V^2NC)$, where $C$ is the number of clusters. For $\theta, C\ll N$, the main complexity in this stage is $\bm{\mathcal{O}}(VNC\log(VN))$. Thus, the main computational complexity of our method is $\bm{\mathcal{O}}(VN\theta d)$, which is linear to the number of samples.

\textbf{\emph{Storage complexity:}} The model involves several variable stores in the optimization process, including $\{\textbf{Y}^{\textrm{(\emph{v})}}\textrm{, }\textbf{D}^{\textrm{(\emph{v})}}\}_{\textrm{\emph{v} = 1}}^{\textrm{\emph{V}}}$, tensor $\bm{\mathcal{J}}$ and $\bm{\mathcal{Y}}$. According to the size of these variables and considering the number of views, their storage complexity is $\bm{\mathcal{O}}(\textrm{\emph{VNC}})$, $\bm{\mathcal{O}}(\textrm{\emph{VN}}^{\textrm{2}})$, $\bm{\mathcal{O}}(\textrm{\emph{VNC}})$, and $\bm{\mathcal{O}}(\textrm{\emph{VNC}})$, respectively. Thus, the total storage complexity of the proposed method is $\bm{\mathcal{O}}(\textrm{\emph{3VNC}}+\textrm{\emph{VN}}^{\textrm{2}})$.

\subsection{Class Equilibrium Analysis}

In this section, we have shown that our model (\ref{objfinal}) can guarantee that the numbers of samples of all clusters are approximately identical, i.e, $N_1=N_2=\cdots=N_c$.

In the first term in the model (\ref{objfinal}), all views are independent, thus for the sake of description, we ignore the variable $v$ and rewrite the first term as
\begin{equation}\label{EA-1}
\begin{array}{l}
tr({{\bf{Y}}^T}{\bf{D}}{\bf{Y}}) = \sum\limits_{i,l} {\left\| {{{\bf{x}}_i} - {{\bf{x}}_j}} \right\|_2^2\left\langle {{{\bf{y}}^i},{{\bf{y}}^j}} \right\rangle } \\
 = \sum\limits_{k = 1}^c {\sum\limits_{\scriptstyle{{\bf{x}}_i} \in {{\bf{A}}_k}\hfill\atop
\scriptstyle{{\bf{x}}_j} \in {{\bf{A}}_k}\hfill} {{{\bf{d}}_{ij}}} }
\end{array}
\end{equation}

According to (\ref{consD}), we have the distances between data points within the same cluster are approximately identical, which can be denoted by constant ${{\overline {\overline {{\bf{d}}}}}}$ in our model, thus, the model (\ref{EA-1}) can be approximately rewritten as
\begin{equation}\label{EA-2}
\begin{array}{l}
tr({{\bf{Y}}^T}{\bf{D}}{\bf{Y}}) = \sum\limits_{k = 1}^c {{\overline {\overline {{\bf{d}}}}}{N_k}({N_k} - 1)} \\
 \ \ \ \ \ \ \ \ \ \ \ \ \ \ \ \ \ \ \ = \sum\limits_{k = 1}^c {N_k^2{\overline {\overline {{\bf{d}}}}}}  - \sum\limits_{k = 1}^c {{\overline {\overline {{\bf{d}}}}}} {N_k}\\
 \ \ \ \ \ \ \ \ \ \ \ \ \ \ \ \ \ \ \ = {\overline {\overline {{\bf{d}}}}}\sum\limits_{k = 1}^c {N_k^2}  - {\overline {\overline {{\bf{d}}}}}N
\end{array}
\end{equation}

Let ${\bf{b}} = {[1,1,...,1]^T} \in {R^c},{\bf{a}} = [{N_1},{N_2},...,{N_c}]^T$, according to cauchy-schwarz inequality, we have
\begin{equation}\label{EA-3}
{\left\langle {{\bf{a}},{\bf{b}}} \right\rangle ^2} \le {\left\| {\bf{a}} \right\|^2}{\left\| {\bf{b}} \right\|^2}
\end{equation}
then,
\begin{equation}\label{EA-4}
\begin{array}{l}
N_1^2 + N_2^2 + ... + N_c^2 \ge \frac{{{{({N_1} + {N_2} + ... + {N_c})}^2}}}{c} = \frac{{{N^2}}}{c}
\end{array}
\end{equation}

If and only if  ${N_1} = {N_2} = ... = {N_c}$, the equation is satisfied.


\subsection{Convergence Analysis of Algorithm \ref{A1}}
	\begin{lemma}[Proposition 6.2 of \cite{lewis2005nonsmooth}]\label{lewis}
		Suppose $F: \mathbb{R}^{n_1\times n_2}\rightarrow \mathbb{R}$ is represented as $F(X)=f \circ \sigma(X)$, where $X\in\mathbb{R}^{n_1\times n_2} $ with SVD
		$X=U \mathrm{diag}(\sigma_1, \ldots, \sigma_n) V^{\mathrm{T}}$, $n=\min(n_1, n_2)$, and $f$ is differentiable. The gradient of $F(X)$ at $X$ is
		\begin{equation}
			\label{deritheorem}
			\frac{\partial F(X)}{\partial X}=U \mathrm{diag}(\theta) V^{\mathrm{T}},
		\end{equation}
		where $\theta=\frac{\partial f(y)}{\partial y}|_{y=\sigma (X)}$.
	\end{lemma}
	\begin{theorem}\label{thm1}[Convergence Analysis of \textbf{Algorithm~\ref{A1}}]
		Let $P_{k}=\{\bm{{\mathcal{Y}}}_{k}, \bm{{\mathcal{J}}}_{k}, \bm{{\mathcal{Q}}}_k\},\ 1\leq k< \infty$ in \eqref{objfinal} be a sequence generated
		by \textbf{Algorithm~\ref{A1}}, then
		\begin{enumerate}
			\item  $P_{k}$ is bounded;
			\item  Any accumulation point of $P_{k}$ is a stationary KKT point of \eqref{objfinal}.
		\end{enumerate}
	\end{theorem}
	\subsubsection{Proof of the 1st part}
	To minimize $\mathcal{J}$ at step $k+1$ in \eqref{SolveJ}, the optimal $\mathcal{J}_{k+1}$ needs to satisfy the first-order optimal condition
	$\lambda\nabla_{\bm{\mathcal {J}}}\|\bm{\mathcal {J}}_{k+1}\|^p_{\Sp}+\mu_k(\bm{\mathcal {J}}_{k+1}-\bm{{\mathcal{Y}}}_{k+1}-\dfrac{1}{\mu_k}\bm{{\mathcal{Q}}}_{k})=0.$
	
	Recall that when $0<p<1$, in order to overcome the singularity of $(|\eta|^p)'=p\eta/|\eta|^{2-p}$ near $\eta=0$, we consider for $0<\epsilon\ll 1$ the approximation
	$$
	\partial |\eta|^p\approx\dfrac{p\eta}{\max\{\epsilon^{2-p},|\eta|^{2-p}\}}.
	$$
	Letting $\overline {\bm{\mathcal {J}}}^{(i)}={\overline {\bm{\mathcal {U}}}}^{(i)}\mathrm{diag}\left(\sigma_j(\overline {\bm{\mathcal {J}}}^{(i)})\right){\overline {\bm{\mathcal {V}}}}^{(i)\mathrm{H}},$ then it follows from Defn.~\ref{Sp-norm} and Lemma~\ref{lewis} that \begin{align*}\frac{\partial \|{\overline {\bm{\mathcal {J}}}}^{(i)}\|^p_{\Sp}}{\partial{\overline {\bm{\mathcal {J}}}}^{(i)}}={\overline {\bm{\mathcal {U}}}}^{(i)}\mathrm{diag}\left(\dfrac{p\sigma_j(\overline {\bm{\mathcal {J}}}^{(i)})}{\max\{\epsilon^{2-p},|\sigma_j(\overline {\bm{\mathcal {J}}}^{(i)})|^{2-p}\}}\right){\overline {\bm{\mathcal {V}}}}^{(i)\mathrm{H}}.\end{align*}
	And  then one can obtain
	\begin{align*}
		&\dfrac{p\sigma_j(\overline {\bm{\mathcal {J}}}^{(i)})}{\max\{\epsilon^{2-p},|\sigma_j(\overline {\bm{\mathcal {J}}}^{(i)})|^{2-p}\}}\leq \dfrac{p}{\epsilon^{1-p}}\\&\Longrightarrow \left\|\frac{\partial \|{\overline {\bm{\mathcal {J}}}}^{(i)}\|^p_{\Sp}}{\partial{\overline {\bm{\mathcal {J}}}}^{(i)}}\right\|^2_F\leq \sum^{N}_{i=1} \dfrac{p^2}{\epsilon^{2(1-p)}}.
	\end{align*}So
	$\frac{\partial \|{\overline {\bm{\mathcal {J}}}}\|^p_{\Sp}}{\partial{\overline {\bm{\mathcal {J}}}}}$
	is bounded.
	
	Let us denote $\widetilde{\mathbf{F}}_{V} = \frac{1}{\sqrt{V}}\mathbf{F}_{V}$, $\mathbf{F}_{V}$ is the discrete Fourier transform matrix of size $V\times V$,\  $\mathbf{F}^{\mathrm{H}}_{V}$ denotes its conjugate transpose. For $\bm{\mathcal {J}}=\overline {\bm{\mathcal {J}}}\times_3 \widetilde{\mathbf{F}}_{V}$ and using the chain rule in matrix calculus, one can obtain that $$\nabla_{\bm{\mathcal {J}}}\|\bm{\mathcal {J}}\|^p_{\Sp}=\frac{\partial \|{\bm{\mathcal {J}}}\|^p_{\Sp}}{\partial{\overline {\bm{\mathcal {J}}}}}\times_3 \widetilde{\mathbf{F}}_{V}^{\mathrm{H}}$$ is bounded.
	
	And it follows that
	\begin{align*}
		&\bm{{\mathcal{Q}}}_{k+1}=\bm{{\mathcal{Q}}}_{k}+\mu_{k}(\bm{{\mathcal{Y}}}_{k+1}-\bm{\mathcal {J}}_{k+1})\\&\Longrightarrow \lambda\nabla_{\bm{\mathcal {J}}}\|\bm{\mathcal {J}}_{k+1}\|^p_{\Sp}=\bm{{\mathcal{Q}}}_{k+1},
	\end{align*}
	$\{\bm{{\mathcal{Q}}}_{k+1}\}$ appears to be bounded.
	
	Moreover, by using the updating rule $$\bm{{\mathcal{Q}}}_{k}=\bm{{\mathcal{Q}}}_{k-1}+\mu_{k-1}(\bm{\mathcal {Y}}_{k}-\bm{\mathcal {J}}_{k}),$$  we can deduce
	\begin{align}
		\label{eq:Lk_ieq}
		&\mathcal{L}_{\mu_k} \left(\bm{{\mathcal{Y}}}_{k+1}, \bm{{\mathcal{J}}}_{k+1}, \bm{{\mathcal{Q}}}_k\right) \leq  \mathcal{L}_{\mu_k} \left(\bm{{\mathcal{Y}}}_{k}, \bm{{\mathcal{J}}}_{k}, \bm{{\mathcal{Q}}}_k\right) \\&= \mathcal{L}_{\mu_{k-1}} \left(\bm{{\mathcal{Y}}}_{k}, \bm{{\mathcal{J}}}_{k}; \bm{{\mathcal{Q}}}_{k-1}\right)\nonumber\\
		&+\frac{\mu_k+\mu_{k-1}}{2\mu^2_{k-1}}\|\bm{{\mathcal{Q}}}_k-\bm{{\mathcal{Q}}}_{k-1}\|_F^2+ \frac{\|\bm{{\mathcal{Q}}}_{k}\|_F^2}{2\mu_k}- \frac{\|\bm{{\mathcal{Q}}}_{k-1}\|_F^2}{2\mu_{k-1}}.\nonumber
	\end{align}
	Thus, summing two sides of \eqref{eq:Lk_ieq} from $k=1$ to $n$, we have
	\begin{equation}
		\begin{aligned}
			&\mathcal{L}_{\mu_n} \left(\bm{{\mathcal{Y}}}_{n+1}, \bm{{\mathcal{J}}}_{n+1}, \bm{{\mathcal{Q}}}_n\right) \leq \mathcal{L}_{\mu_{0}} \left(\bm{{\mathcal{Y}}}_{1}, \bm{{\mathcal{J}}}_{1}, \bm{{\mathcal{Q}}}_{0}\right) \\
			+&\frac{\|\bm{{\mathcal{Q}}}_{n}\|_F^2}{2\mu_n}- \frac{\|\bm{{\mathcal{Q}}}_{0}\|_F^2}{2\mu_{0}}+\sum_{k=1}^n\left(\frac{\mu_k+\mu_{k-1}}{2\mu^2_{k-1}}\|\bm{{\mathcal{Q}}}_k-\bm{{\mathcal{Q}}}_{k-1}\|_F^2\right).
			\label{eq:Lk_sum}
		\end{aligned}
	\end{equation}
	Observe that
	\[
	\sum_{k=1}^{\infty}\frac{\mu_k+\mu_{k-1}}{2\mu_{k-1}^2}<\infty,
	\]
	we have the right-hand side of \eqref{eq:Lk_sum} is finite and thus $\mathcal{L}_{\mu_n} \left(\bm{{\mathcal{Y}}}_{n+1}, \bm{{\mathcal{J}}}_{n+1}, \bm{{\mathcal{Q}}}_n\right)$ is bounded. Notice
	\begin{align}
		\label{eq:Ln_bdd}
		&\mathcal{L}_{\kappa_n} \left(\bm{{\mathcal{Y}}}_{n+1}, \bm{{\mathcal{J}}}_{n+1}, \bm{{\mathcal{Q}}}_n\right) =\sum_{v=1}^V\mathrm{tr}(\textbf{Y}_{n+1}^{(v)}\textbf{D}^{(v)}\textbf{Y}_{n+1}^{(v)\mathrm{T}})\nonumber\\
		&+\lambda\|\bm{\mathcal {J}}_{n+1}\|^p_{\Sp} + \frac{\mu_{n}}{2}\|\bm{\mathcal {Y}}_{n+1}-\bm{\mathcal {J}}_{n+1}+\frac{\bm{{\mathcal{Q}}}_{n}}{\mu_{n}}\|_F^2,
	\end{align}
	and each term  of \eqref{eq:Ln_bdd} is nonnegative, following from the boundedness of $\mathcal{L}_{\mu_n} \left(\bm{{\mathcal{Y}}}_{n+1}, \bm{{\mathcal{J}}}_{n+1}, \bm{{\mathcal{Q}}}_n\right)$, we can deduce each term of \eqref{eq:Ln_bdd} is bounded. And
	$\|\bm{{\mathcal{J}}}_{n+1}\|^p_{\Sp}$ being bounded implies that all singular values of $\bm{{\mathcal{J}}}_{n+1}$ are bounded and hence $\|\bm{{\mathcal{J}}}_{n+1}\|^2_F$ (the sum of squares of singular values) is  bounded. Therefore, the sequence $\{\bm{{\mathcal{J}}}_k\}$ is  bounded. Considering the positive semi-definiteness of matrix $\textbf{D}^{(v)}$, we can deduce $\textbf{Y}^{(v)}_{n+1}$ is bounded and thus the sequence $\{\bm{{\mathcal{Y}}}_k\}$ is  bounded.
	
	\subsubsection{Proof of the 2nd part}
	
	From Weierstrass-Bolzano theorem, there exists at least one accumulation point of the sequence $P_{k}$. We denote one of the points $P^*=\{\bm{{\mathcal{Y}}}^*, \bm{{\mathcal{J}}}^*, \bm{{\mathcal{Q}}}^*\}$. Without loss of generality, we assume $\{\mathcal{P}_{k}\}^{+\infty}_{k=1}$ converge to $P^*.$
	
	Note that from the updating rule for $\bm{{\mathcal{Q}}}$, we have $$\bm{{\mathcal{Q}}}_{k+1}=\bm{{\mathcal{Q}}}_{k}+\mu_{k}(\bm{{\mathcal{Y}}}_{k}-\bm{{\mathcal{J}}}_{k})\Longrightarrow \bm{{\mathcal{J}}}^*=\bm{{\mathcal{Y}}}^*.$$

	In the $\bm{\mathcal {J}}$-subproblem, we have $$\lambda\nabla_{\bm{\mathcal {J}}}\|\bm{\mathcal {J}}_{k+1}\|^p_{\Sp}=\bm{{\mathcal{Q}}}_{k+1}\Longrightarrow\bm{{\mathcal{Q}}}^*=\lambda\nabla_{\bm{\mathcal {J}}}\|\bm{\mathcal {J}}^*\|^p_{\Sp}.$$
	
	In \eqref{SolvingY}, we have
	$$\alpha_\emph{\textrm{v}}^\emph{\textrm{r}}\partial \mathrm{tr}(\textbf{Y}^{(v)}_{k+1}\textbf{D}^{(v)}\textbf{Y}^{(v)T}_{k+1})-\mu_k(\textbf{J}_{k+1}^{(v)}-\textbf{Y}_{k+1}^{(v)}+\textbf{Q}_{k}^{(v)}/\mu_k)=0.$$
	Now by the updating rule ${\bm{{\mathcal Q}}}{\rm{ = }}{\bm{{\mathcal Q}}}{\rm{ + }}\mu {\rm{(}}{\bm{{\mathcal Y}}}{\rm{ - }}{\bm{{\mathcal J}}}{\rm{)}}$, we can see
	\begin{align*}\label{eq:parLk+1}
		\textbf{Q}_{k+1}^{(v)} = 2\alpha_\emph{\textrm{v}}^\emph{\textrm{r}}\textbf{Y}^{(v)}_{k+1}\textbf{D}^{(v)}\Longrightarrow\textbf{Q}^{(v)*} = 2\alpha_\emph{\textrm{v}}^\emph{\textrm{r}}\textbf{Y}^{(v)*}\textbf{D}^{(v)},
	\end{align*}
	Therefore, one can see that the sequences $\bm{{\mathcal{Y}}}^*, \bm{{\mathcal{J}}}^*, \bm{{\mathcal{Q}}}^*$ satisfy the KKT conditions of the Lagrange function \eqref{objfinal}.

\section{Experiments}
We evaluate our model on two toy datasets and six benchmark datasets through some experiments implemented on a Windows 10 desktop computer with a 2.40GHz Intel Xeon Gold 6240R CPU, 64 GB RAM, and MATLAB R2021a (64-bit).

\subsection{Experiments on A Toy Dataset}

Experiments in this section are conducted to verify the effectiveness of our method for non-linearly separable clusters. First, we construct two toy datasets called two-moon and three-ring, respectively, which contain 2 views each dataset. The two-moon dataset contains 200 samples belonging to 2 non-linearly separable clusters. Each sample is represented by a 2D coordinate, and the second view is obtained by performing FFT on the first view. The three-ring dataset contains 200 samples belonging to 3 non-linearly separable clusters, with each cluster represented by a 2D coordinate. The second view is obtained by performing FFT on the first view. Figure~\ref{result_ConsD} shows the clustering results of our method with the Euclidean distance and the distance constructed by our proposed function (\ref{consD}), respectively. The visualization indicates that the performance of our method  is significantly improved when using our proposed function (\ref{consD}) to calculate the distance between data points, compared with the Euclidean distance. This indicates that our method is capable of separating non-linearly separable clusters well in the input space.


\begin{figure}[!t]
	\centering
	\subfigure[Euclidean distance on three-ring]{
		\includegraphics[width=0.47\linewidth]{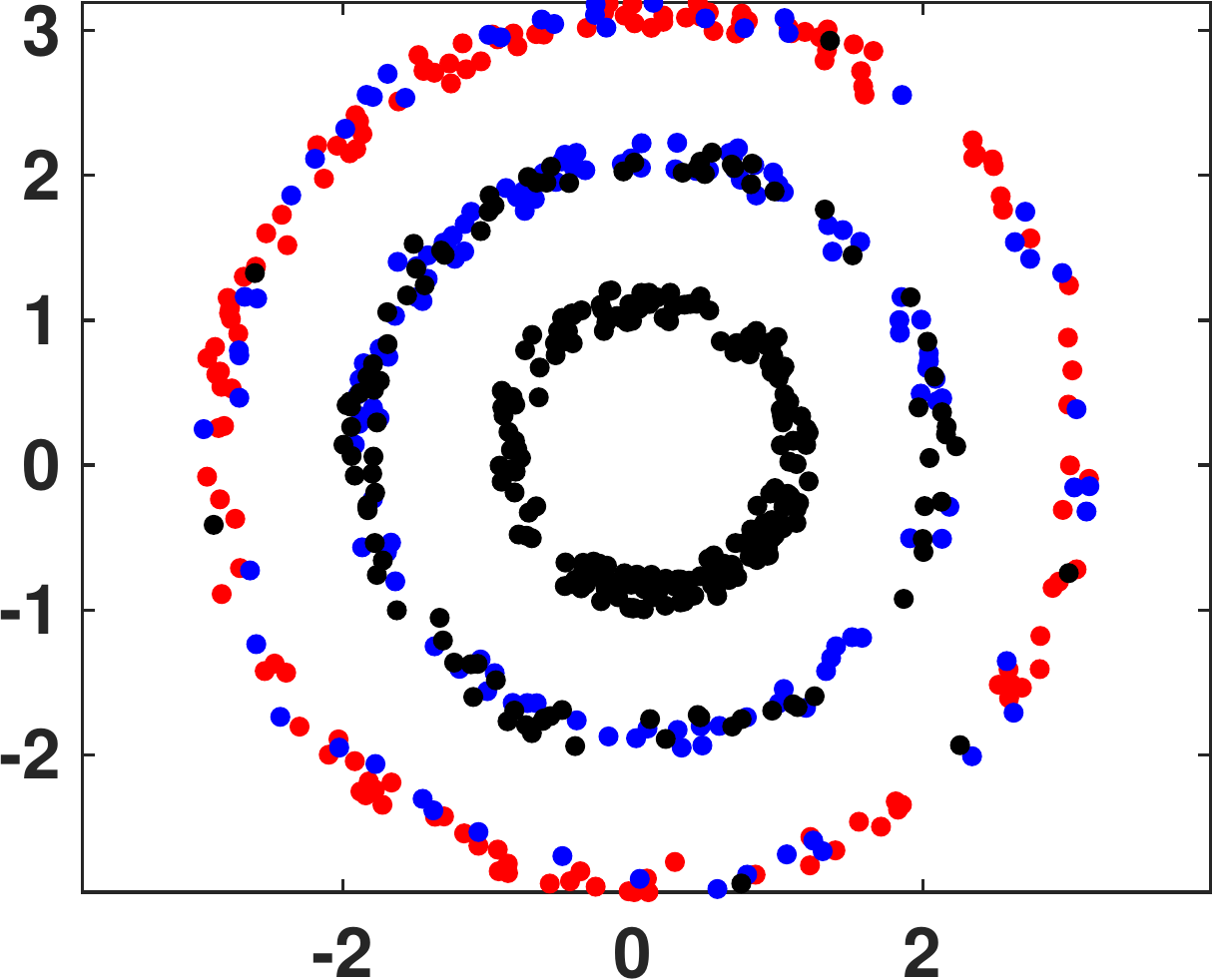}
	}
	\subfigure[Our function on three-ring]{
		\includegraphics[width=0.47\linewidth]{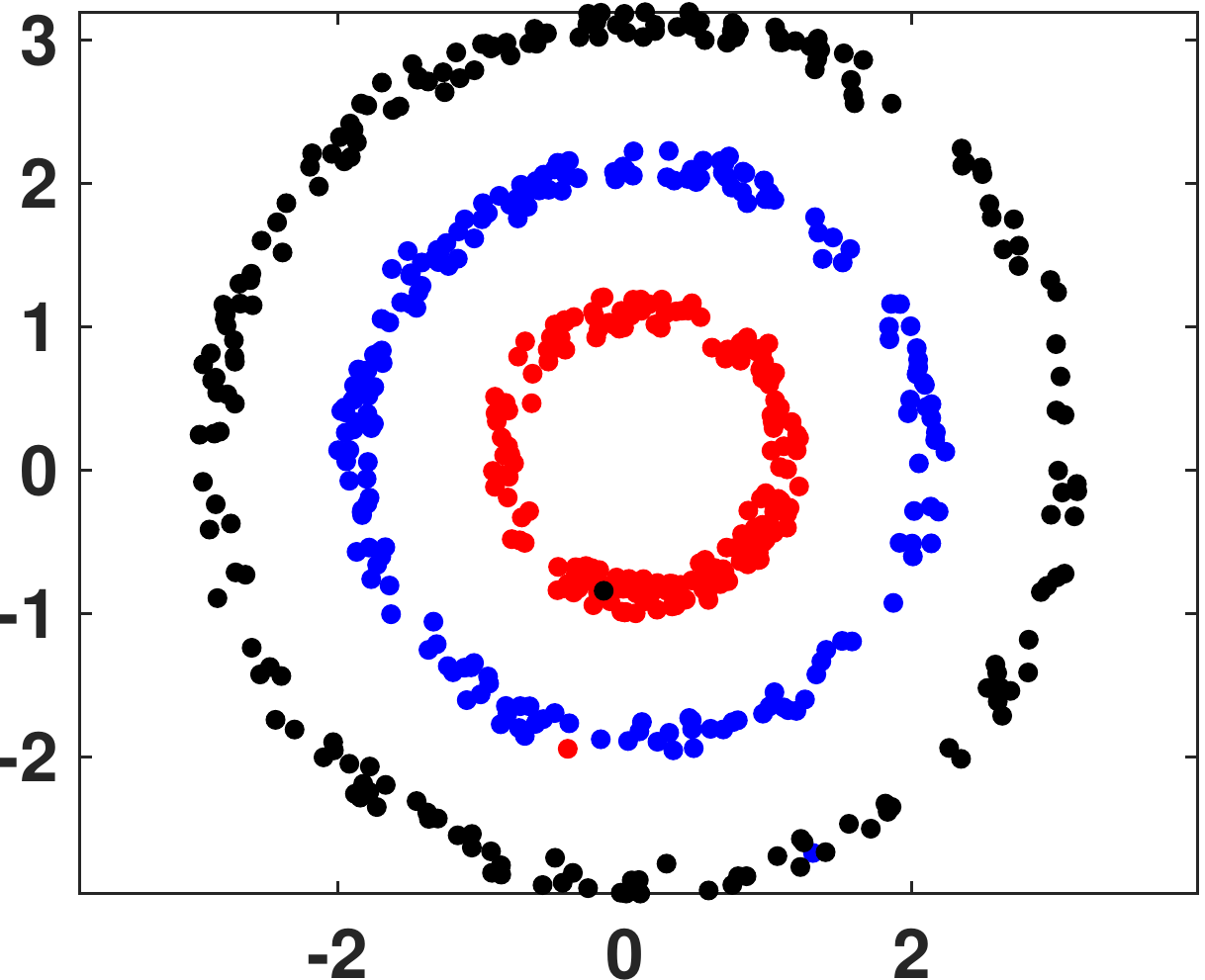}
	}
	\subfigure[Euclidean distance on two-moon]{
		\includegraphics[width=0.47\linewidth]{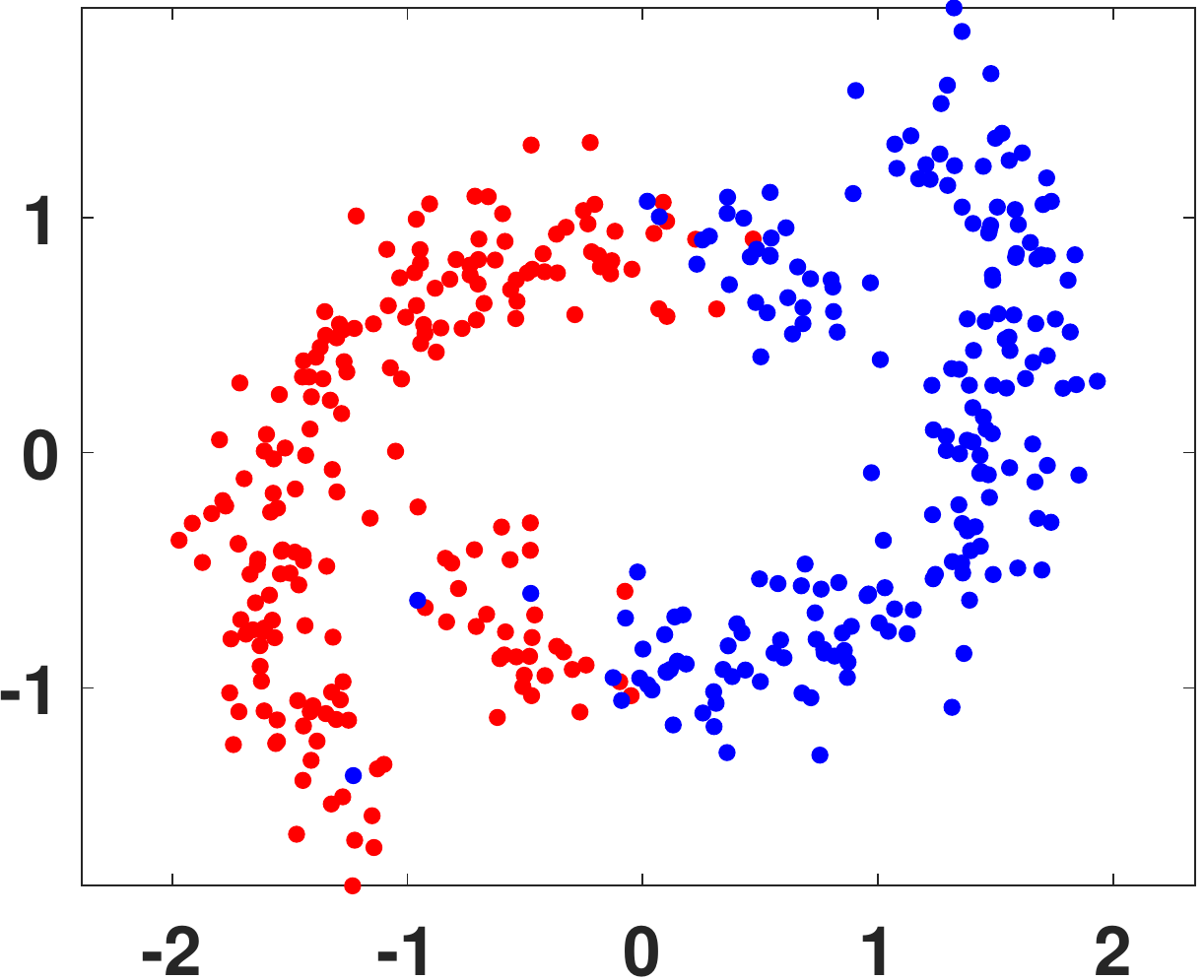}
	}
	\subfigure[Our function on two-moon]{
		\includegraphics[width=0.47\linewidth]{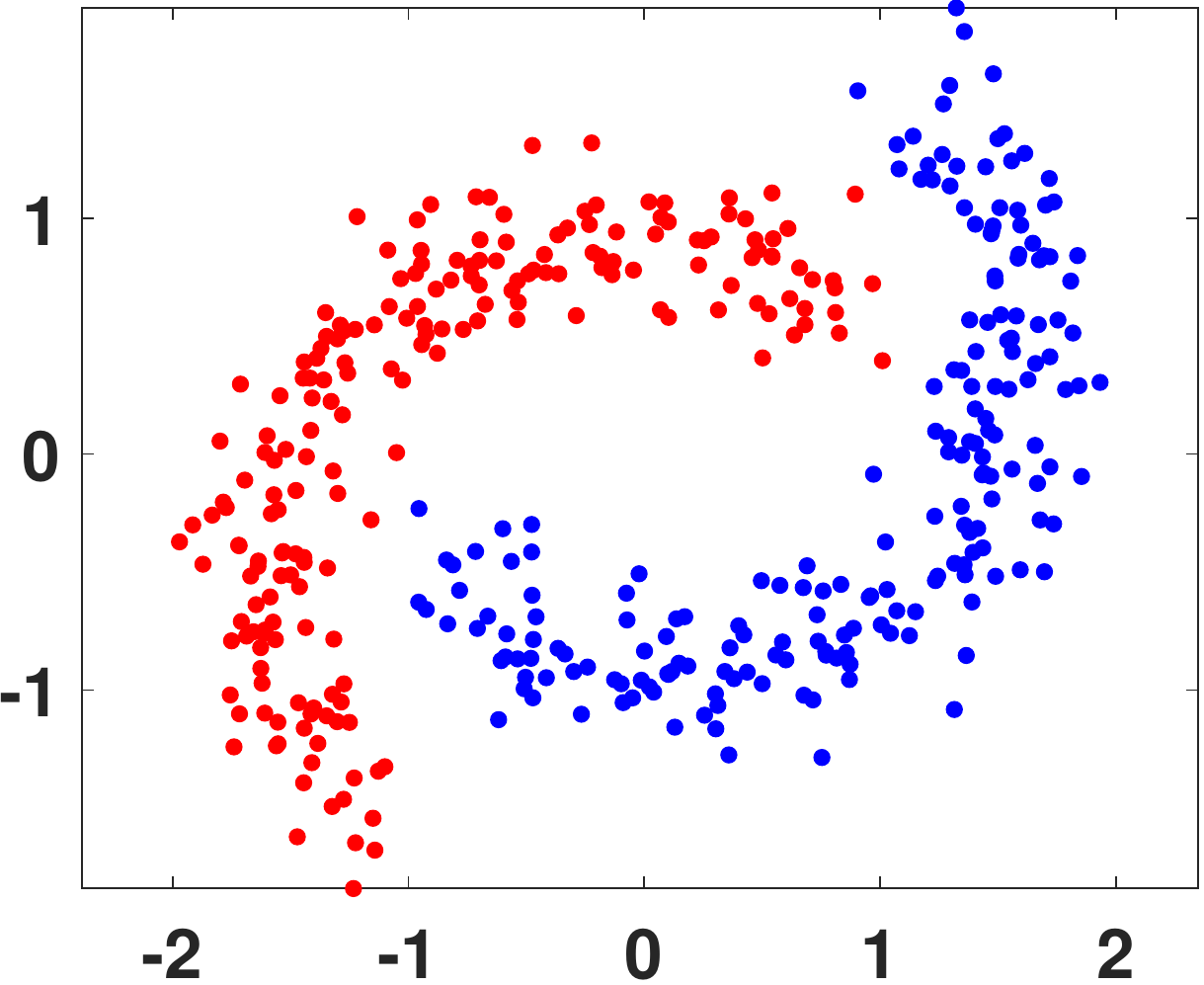}
	}
	\caption{Clustering performance of our proposed method with different distance matrices on two toy datasets.}
	\label{result_ConsD}
\end{figure}

\subsection{Experimental on Benchmark Datasets}
\subsubsection{Datasets:}
We do experiments to evaluate the effectiveness of our method on the following widely used six multi-view benchmark datasets: MSRC~\cite{WinnJ05}, ORL\footnote{\url{http://www.cl.cam.ac.uk/research/dtg/attarchive/facedatabase.html}}, HW~\cite{Dua:2019}, and Mnist4~\cite{Deng12} are common real benchmark datasets, while Reuters~\cite{ApteDW94} and NUS-WIDE~\cite{ChuaTHLLZ09} are two large-scale datasets. Table~\ref{infoData} provides detailed information for these datasets, where the numbers in parentheses indicate feature dimensions.

\begin{table*}[!h]
	\centering
	\begin{center}
		\caption{The information of datasets.}\label{infoData}
			\begin{tabular}{l | c c c c c c}
				\toprule
				Datasets    & MSRC    & ORL    & HW     & Mnist4  & Reuters   & NUS-WIDE  \\\midrule
				Data Type   & images  &images  &images  &images   &documents  &images     \\
				Size        & 210     &400     &2000    &4000     &18758      &30000      \\
				Classes     &7        &40      &10      &4        &6          &31         \\
				View\#1     &CM(24)    &Intensity(4096) &FOU(76)   &ISO(30)  &English(21531)  &CH(64)    \\
				View\#2     &HOG(576)  &LBP(3304)       &FAC(216)  &LDA(9)   &France(24892)   &CM(225)   \\
				View\#3 	&GIST(512) &Gabor(6750)     &ZER(47)   &DNPE(30) &German(34251)   &CORR(144) \\
				View\#4     &LBP(256)  &-               &MOR(6)    &-        &Italian(15506)  &EDH(73)   \\
				View\#5     &CENT(254) &-               &-         &-        &Spanish(11547)  &WT(128)   \\
				\bottomrule
		\end{tabular}
	\end{center}
\end{table*}

\subsubsection{Competing Algorithms:}
The following representative methods are adopted to evaluate the performance of our model:
\begin{itemize}
\item \textbf{SC}~\cite{NgJW01}, Classical single view spectral clustering;
\item \textbf{RMKMC}~\cite{2013MultiR}, Multi-view \emph{k}-means clustering on big data;
\item \textbf{FRMVK}~\cite{YangS19a}, Feature-reduction multi-view \emph{k}-means;
\item \textbf{RDEKM}~\cite{XuHNL17}, Re-weighted discriminatively embedded \emph{k}-means for multi-view clustering;
\item \textbf{MVASM}~\cite{HanXNL22}, Multi-view \emph{k}-means clustering with adaptive sparse memberships and weight allocation;
\item \textbf{WMCFS}~\cite{XuWL16}, Weighted multi-view clustering with feature selection;
\item \textbf{Co-Reg}~\cite{KumarRD11}, Co-regularized multi-view SC;
\item \textbf{CSMSC}~\cite{Luo2018Consistent}, Consistent and specific multi-view subspace clustering;
\item \textbf{MVGL}~\cite{ZhanZGW18}, Graph learning for multi-view clustering;
\item \textbf{LTCPSC}~\cite{XuZXGG20}, Low-rank tensor constrained co-regularized multi-view spectral clustering;
\item \textbf{ETLMSC}~\cite{WuLZ19}, Essential tensor learning for multi-view spectral clustering.
\end{itemize}

\subsubsection{Metrics:}
The widely used 3 metrics are applied to evaluate the performance of our method, (1) Accuracy (ACC)~\cite{CaiHH05}; (2) Normalized Mutual Information (NMI)~\cite{EstevezTPZ09}; (3) Purity~\cite{VarshavskyLH05}. These indicators measure the consistency between the true labels and the predicted labels from different perspectives, and a higher value of indicator means better clustering performance.

\subsubsection{Parameters:}
There are five parameters in our model, $\lambda$, $r$, $p$, and $\Omega$. In the experiments, we set $\lambda\in$[0.1, 1, 5, 10, 50, 100], $p$ from 0.1 to 1 with the interval 0.1, $r$ from 3 to 10 with the interval 1, and $\Omega \in$[0.0001, 0.0005, 0.001, 0.003, 0.005, 0.01, 0.03, 0.05, 0.1]. For the anchor ratio $\theta$, we set $\theta = 0.5$ on the MSRC, ORL, Mnist4, and HW databases, $\theta = 0.1$ on the Reuters database, and $\theta = 0.02$ on the NUS-WIDE database. The small anchor ratio of Reuters and NUS-WIDE are set to reduce memory consumption. Because $\Omega$ is only relevant to the adjacency matrix, we tune $\Omega$ while keeping other parameters fixed, and tune other parameters with fixed $\Omega$.

\subsubsection{Results}
We do experiments on six datasets and record the metrics comparison in Tables.~\ref{result12} and~\ref{result56}. SC (best) records the best results among all views. The following can be observed:

First, most multi-view clustering methods achieve better results than single-view SC. It is because the information embedded in different views is complementary and the multi-view methods use this information well. Second, compared with the multi-view $k$-means RMKMC, FRMVK, RDEKM, and MVASM, our method achieves better performance. What's more, the results of our method are better than the state-of-the-art multi-view spectral clustering algorithms. This is because our method is suitable for both linearly and non-linearly separate clusters in the input space, takes full advantage of the complementary information among multi-views, and the considers of the difference among different views. Finally, on two large-scale datasets Reuters and NUS-WIDE, some methods such as RDEKM and CSMSC cannot work on computers with small memory due to the out-of-memory. However, in this case, our method can work and obtain great results by setting a low anchor rate, as shown in Table.~\ref{result56}, especially in NUS-WIDE, our method significantly and consistently outperforms all competitors, which demonstrates the effectiveness on large-scale datasets.

\begin{table*}[!t]
	\centering
	\caption{The results on MSRC, HW, ORL, and Mnist4 datasets.}\label{result12}
	\begin{center}
			\begin{tabular}{l | l l l | l l l | l l l | l l l}
				\toprule
				Datasets &\multicolumn{3}{c|}{MSRC} &\multicolumn{3}{c|}{HW} &\multicolumn{3}{c|}{ORL} &\multicolumn{3}{c}{Mnist4}\\
				\midrule
				Methods & ACC & NMI & Purity & ACC & NMI & Purity & ACC & NMI & Purity & ACC & NMI & Purity \\
				\midrule
				SC (best)&0.663 &0.534 &0.675 &0.639 &0.616 &0.653 &0.727 &0.868 &0.762 &0.713 &0.558	&0.713 \\
				RMKMC &0.700 &0.604 &0.700 &0.804 &0.785 &0.839 &0.543 &0.749 &0.620 &0.895 &0.739 &0.895 \\
				FRMVK &0.391 &0.292 &0.391 &0.790 &0.777 &0.806 &0.368 &0.589 &0.383 &0.882 &0.715 &0.882\\
				RDEKM &0.900 &0.808 &0.900 &0.764 &0.755 &0.793 &0.783 &0.578 &0.765 &0.715 &0.881 &0.714\\
				MVASM &0.604 &0.619 &0.644 &0.726 &0.842 &0.644 &0.393 &0.575 &0.539 &0.892 &0.734 &0.892 \\
				WMCFS &0.566 &0.600 &0.619 &0.679 &0.590 &0.641 &0.630 &0.787 &0.690 &0.892 &0.729 &0.892 \\
				CSMSC &0.758 &0.735 &0.793 &0.806 &0.793 &0.867 &0.857 &0.935 &0.882 &0.643 &0.645 &0.832 \\
				MVGL  &0.690 &0.663 &0.733 &0.811 &0.809 &0.831 &0.765 &0.871 &0.815 &0.912 &0.785 &0.910 \\
				Co-Reg&0.635 &0.578 &0.659 &0.784 &0.758 &0.795 &0.668 &0.824 &0.713 &0.785 &0.602 &0.786	\\
				ETLMSC&0.962 &0.937 &0.962 &0.938 &0.893 &0.938 &0.958 &0.991 &0.970 &0.934 &0.847 &0.934 \\
				LTCPS &0.981 &0.957 &0.981 &0.920 &0.869 &0.920 &0.981 &0.994 &0.983 &0.929 &0.813 &0.929 \\
				Ours  &\textbf{1.000}   &\textbf{1.000}    &\textbf{1.000}   &\textbf{0.994}  &\textbf{0.986}  &\textbf{0.994} &\textbf{1.000 }   &\textbf{1.000}    &\textbf{1.000}   &\textbf{0.992}  &\textbf{0.969}  &\textbf{0.992} \\
				\bottomrule
		\end{tabular}
	\end{center}
\end{table*}

\begin{table}[!t]
	\centering
	\caption{The results on Reuters and NUS-WIDE datasets.}\label{result56}
	\begin{center}
		\resizebox{1.0\columnwidth}{!}{
			\begin{tabular}{l| l l l| l l l}
				\toprule
				Datasets &\multicolumn{3}{c|}{Reuters} &\multicolumn{3}{c}{NUS-WIDE}\\
				\midrule
				Methods  &ACC   &NMI   &Purity &ACC  &NMI   &Purity \\
				\midrule
				SC (best)&0.269 &0.002 &0.272 &0.131 &0.019 &0.140  \\
				RMKMC    &0.422 &0.259 &0.531 &0.125 &0.123 &0.221  \\				
				FRMVK    &0.272 &0.000 &0.272 &0.124 &0.095 &0.200  \\
				RDEKM    &OM    &OM    &OM    &OM    &OM    &OM     \\
				MVASM    &0.458 &0.136 &0.814 &0.141 &0.096 &0.254  \\
				WMCFS    &0.427 &0.278 &0.530 &OM    &OM    &OM     \\
				CSMSC    &OM    &OM    &OM    &OM    &OM    &OM     \\
				MVGL     &0.271 &0.021 &0.281 &OM    &OM    &OM     \\
				Co-Reg   &0.563 &0.326 &0.552 &0.119 &0.114 &0.214  \\
				ETLMSC   &OM    &OM    &OM    &OM    &OM    &OM     \\
				LTCPS    &OM    &OM    &OM    &OM    &OM    &OM     \\
				Ours  &\textbf{0.682 } &\textbf{0.638} &\textbf{0.824} &\textbf{0.510} &\textbf{0.708} &\textbf{0.673} \\
				\bottomrule
		\end{tabular}}
	\end{center}
\end{table}

\begin{figure}[!t]
	\centering
	\subfigure[MSRC]{
		\begin{minipage}[t]{0.47\linewidth}
			\centering
			\includegraphics[width=1.0\linewidth]{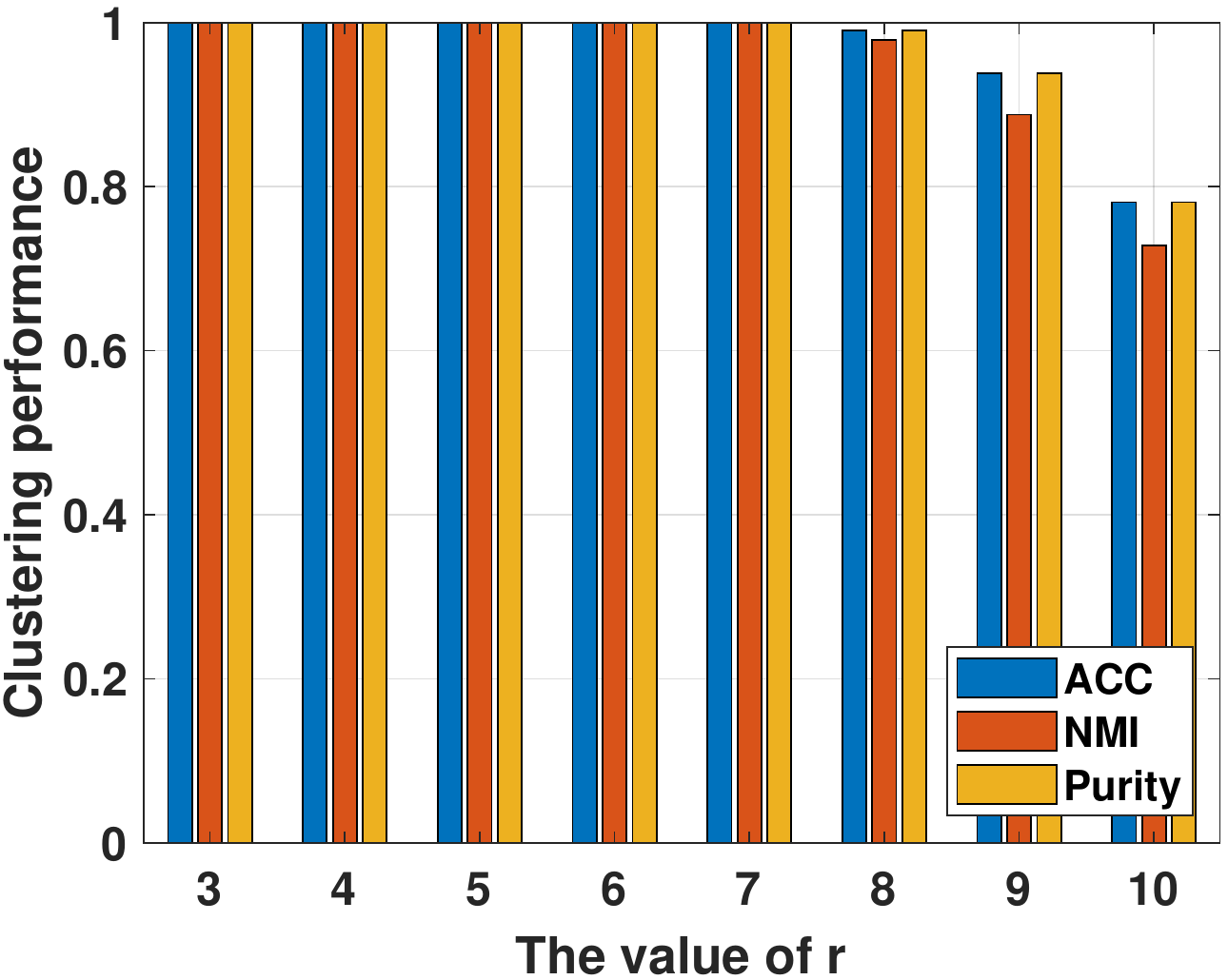}
		\end{minipage}
	}
	\subfigure[ORL]{
		\begin{minipage}[t]{0.47\linewidth}
			\centering
			\includegraphics[width=1.0\linewidth]{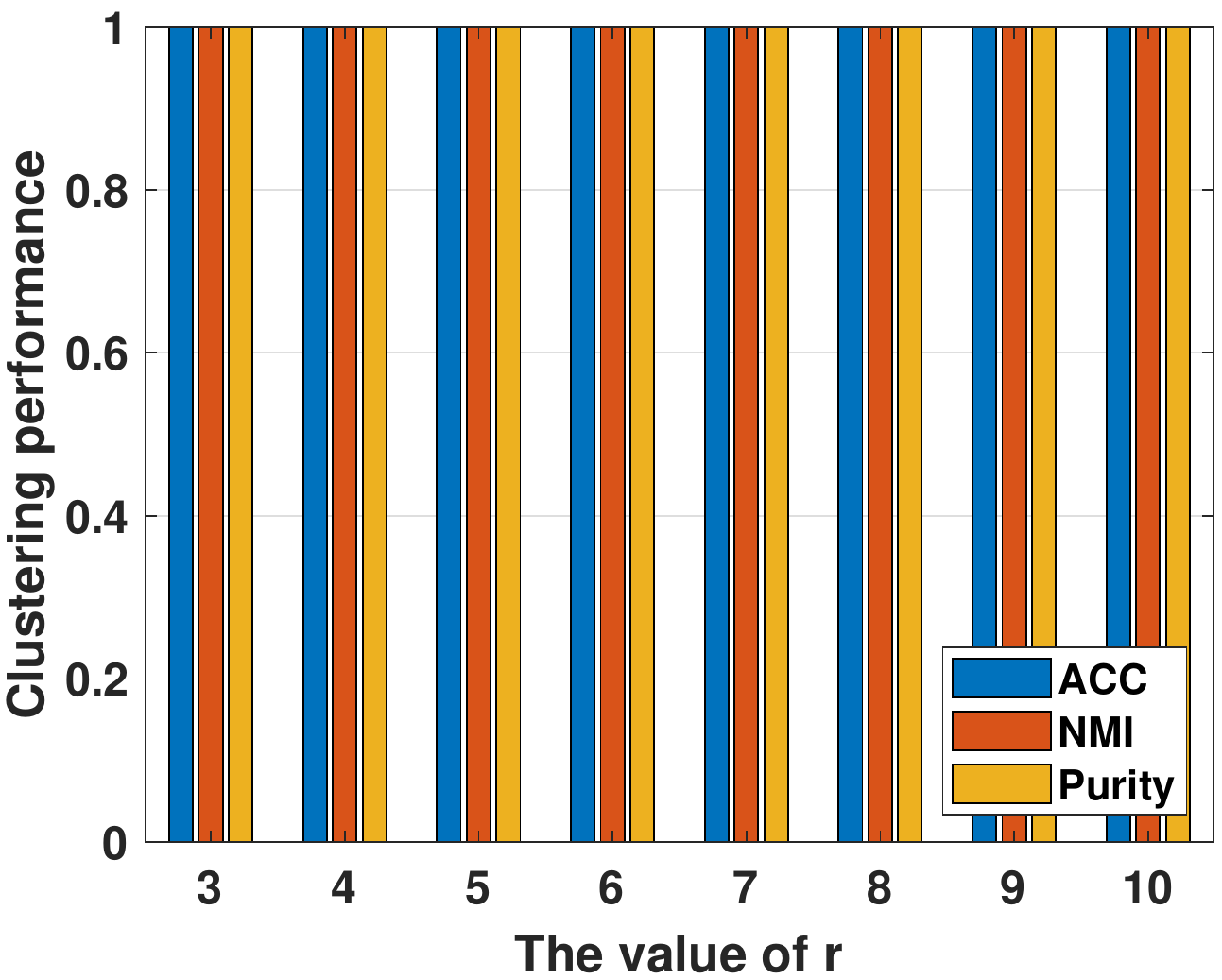}
		\end{minipage}
	}
	\subfigure[HW]{
		\begin{minipage}[t]{0.47\linewidth}
			\centering
			\includegraphics[width=1.0\linewidth]{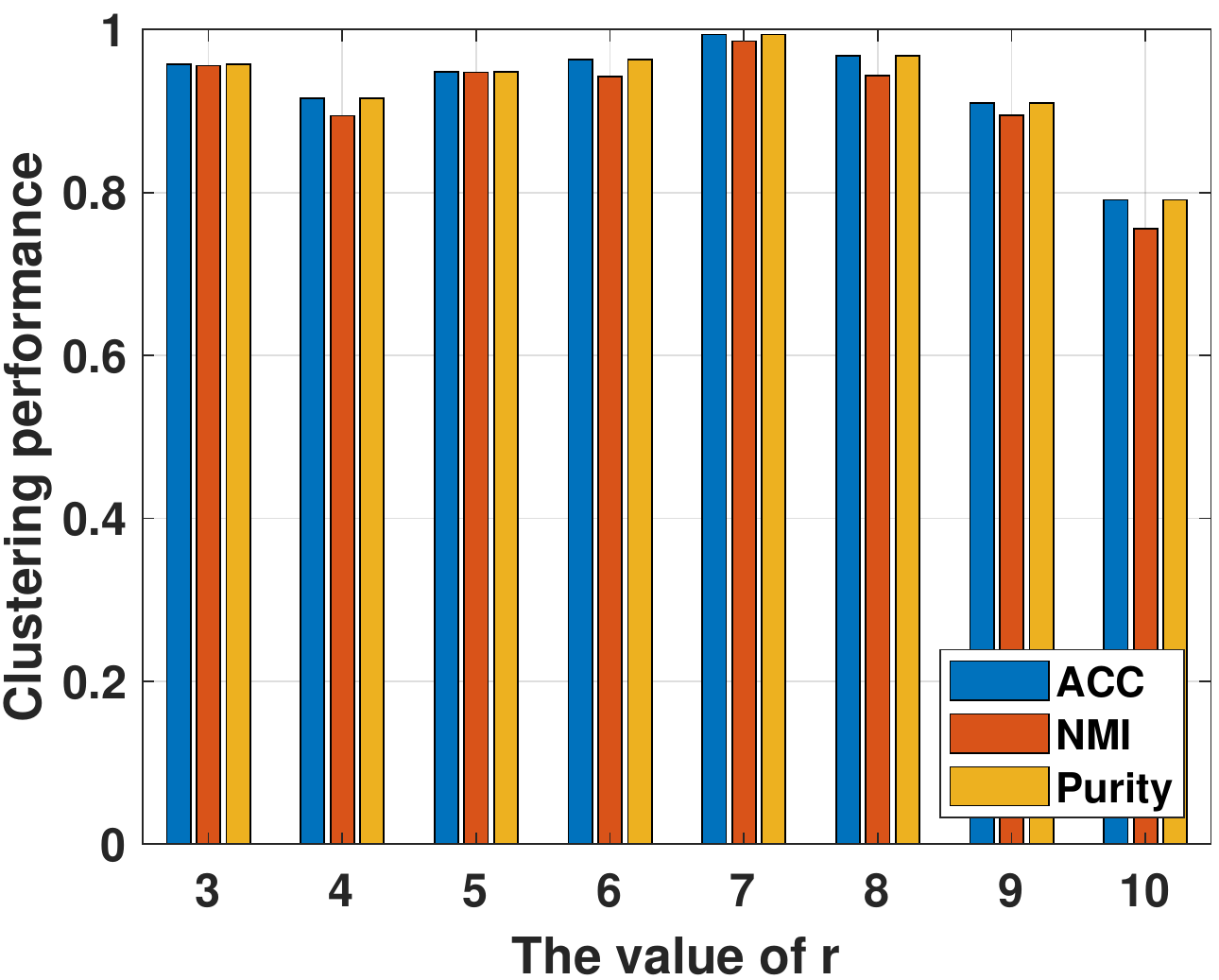}
		\end{minipage}
	}
	\subfigure[Mnist4]{
		\begin{minipage}[t]{0.47\linewidth}
			\centering
			\includegraphics[width=1.0\linewidth]{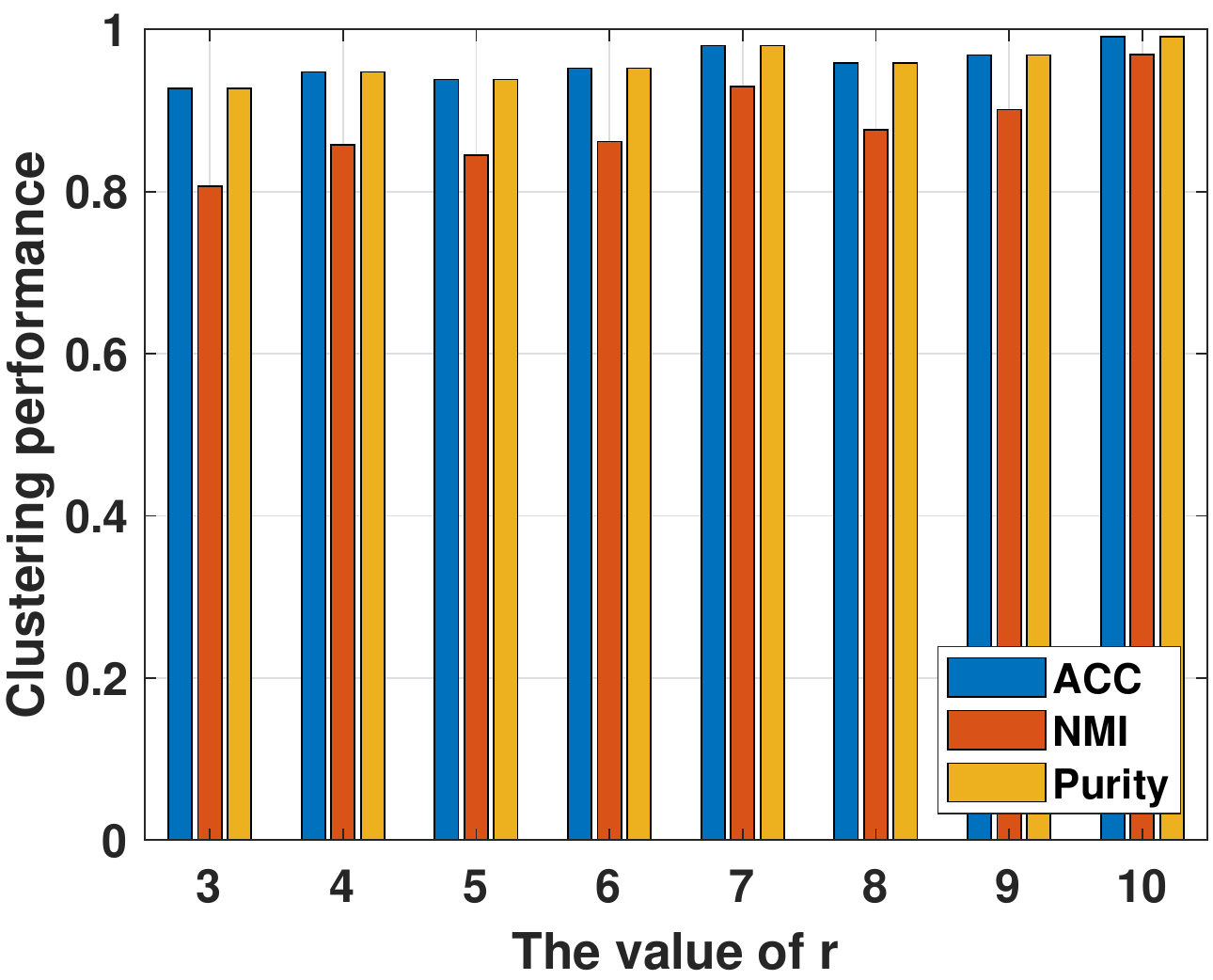}
		\end{minipage}
	}
	\centering
	\caption{Clustering performance vs. $r$ on four datasets.}
	\label{result_R}
\end{figure}

\begin{figure}[!t]
	\centering
	\subfigure[MSRC]{
		\begin{minipage}[t]{0.47\linewidth}
			\centering
			\includegraphics[width=1.0\linewidth]{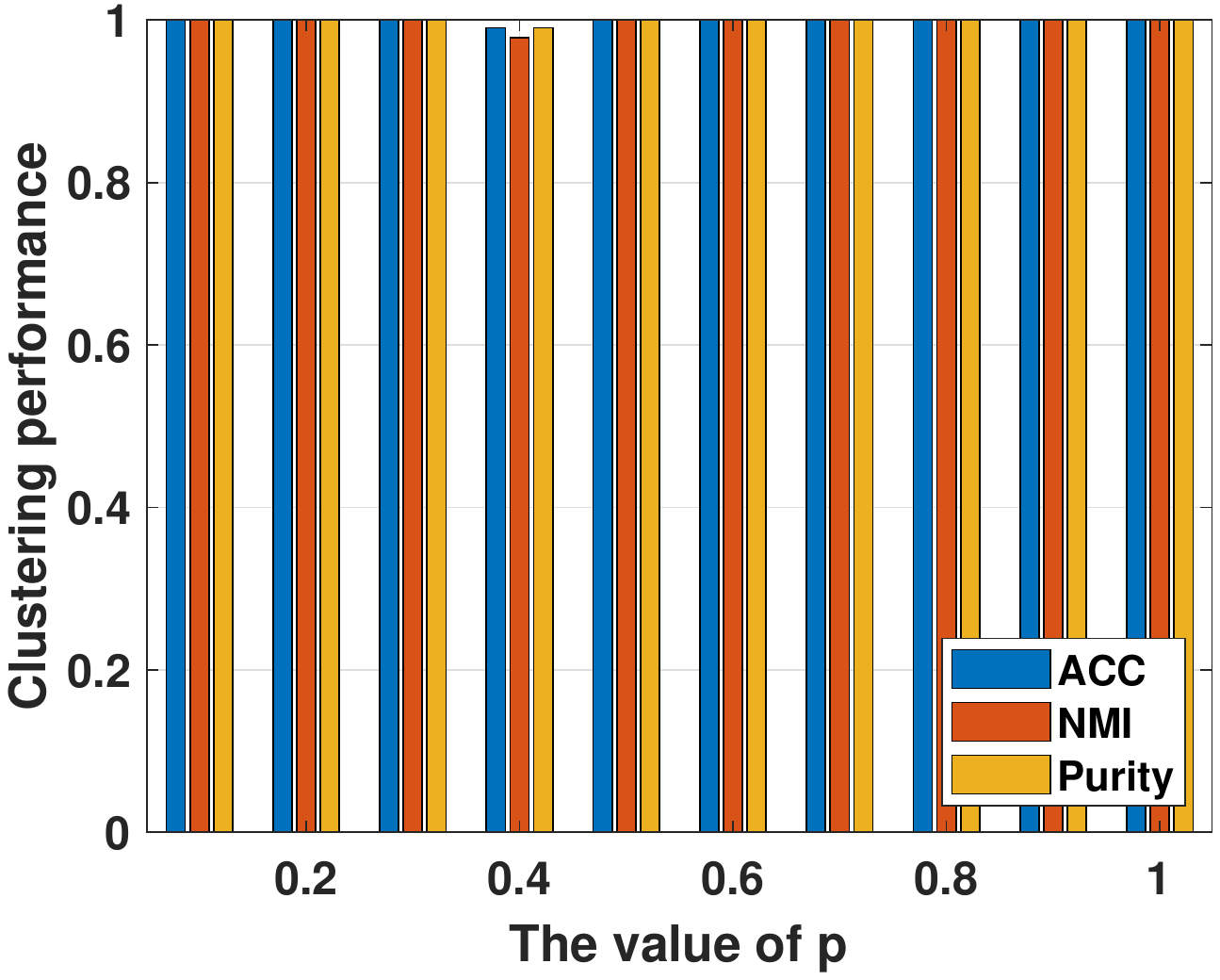}
		\end{minipage}
	}
	\subfigure[ORL]{
		\begin{minipage}[t]{0.47\linewidth}
			\centering
			\includegraphics[width=1.0\linewidth]{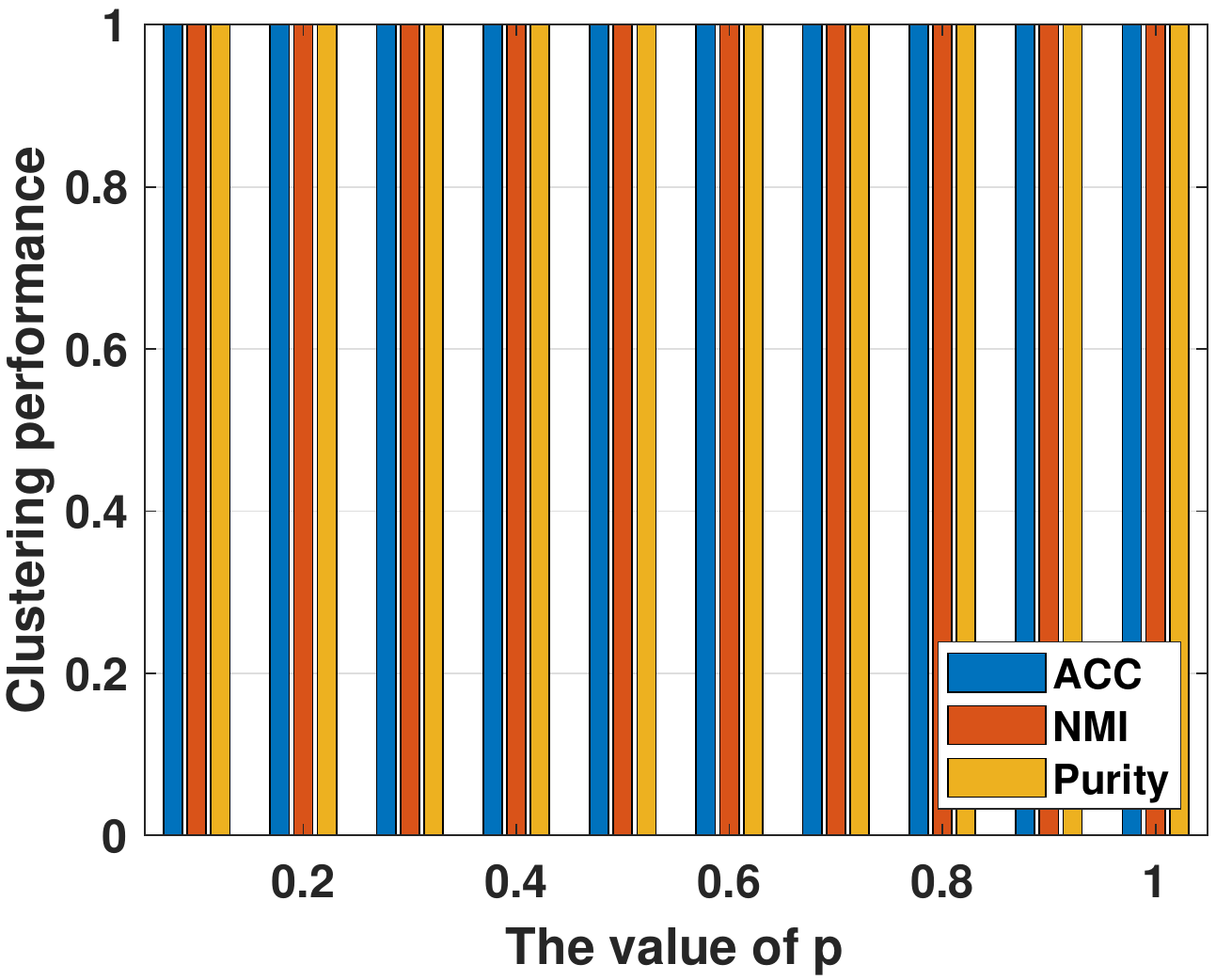}
		\end{minipage}
	}
	\subfigure[HW]{
		\begin{minipage}[t]{0.47\linewidth}
			\centering
			\includegraphics[width=1.0\linewidth]{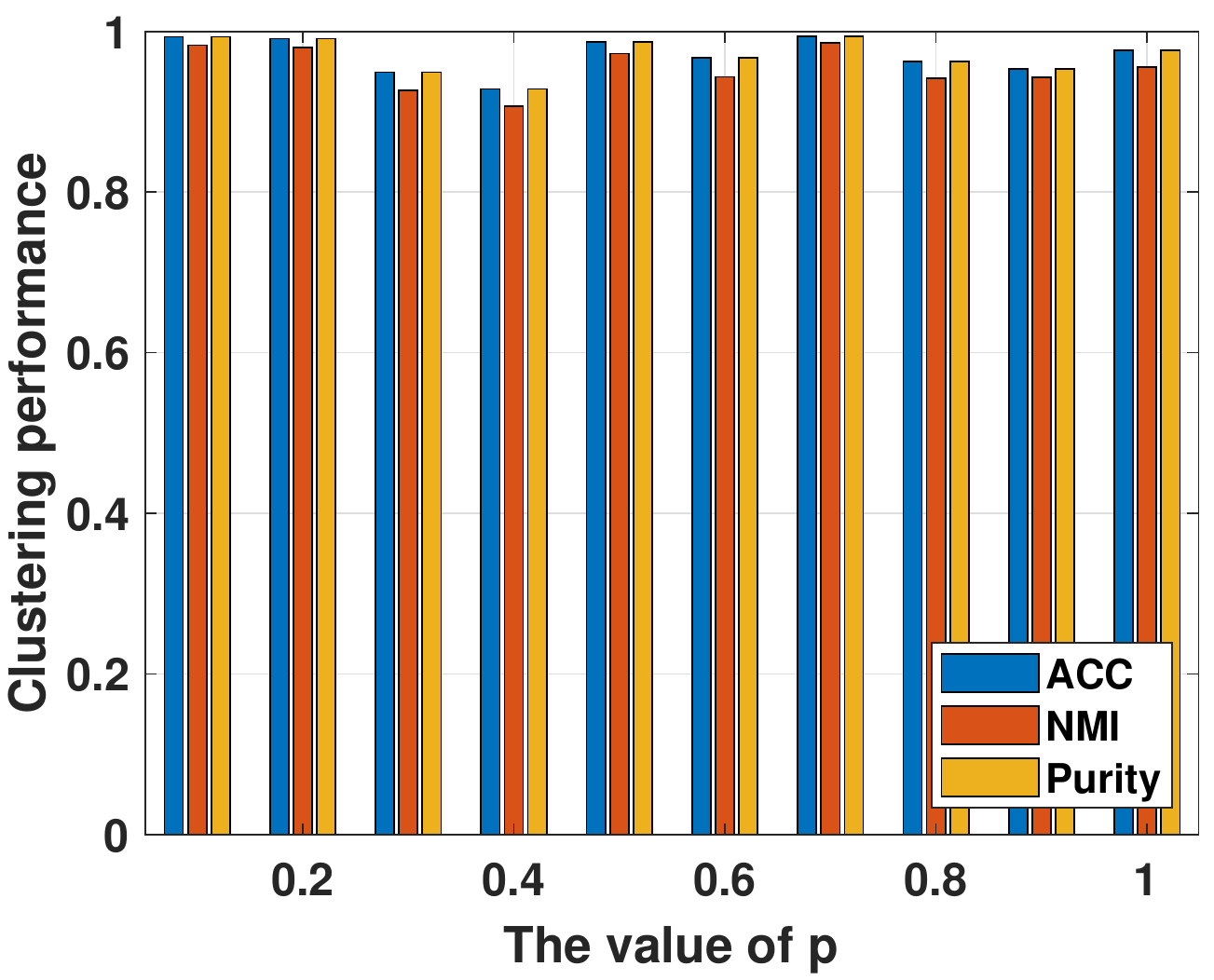}
		\end{minipage}
	}
	\subfigure[Mnist4]{
		\begin{minipage}[t]{0.47\linewidth}
			\centering
			\includegraphics[width=1.0\linewidth]{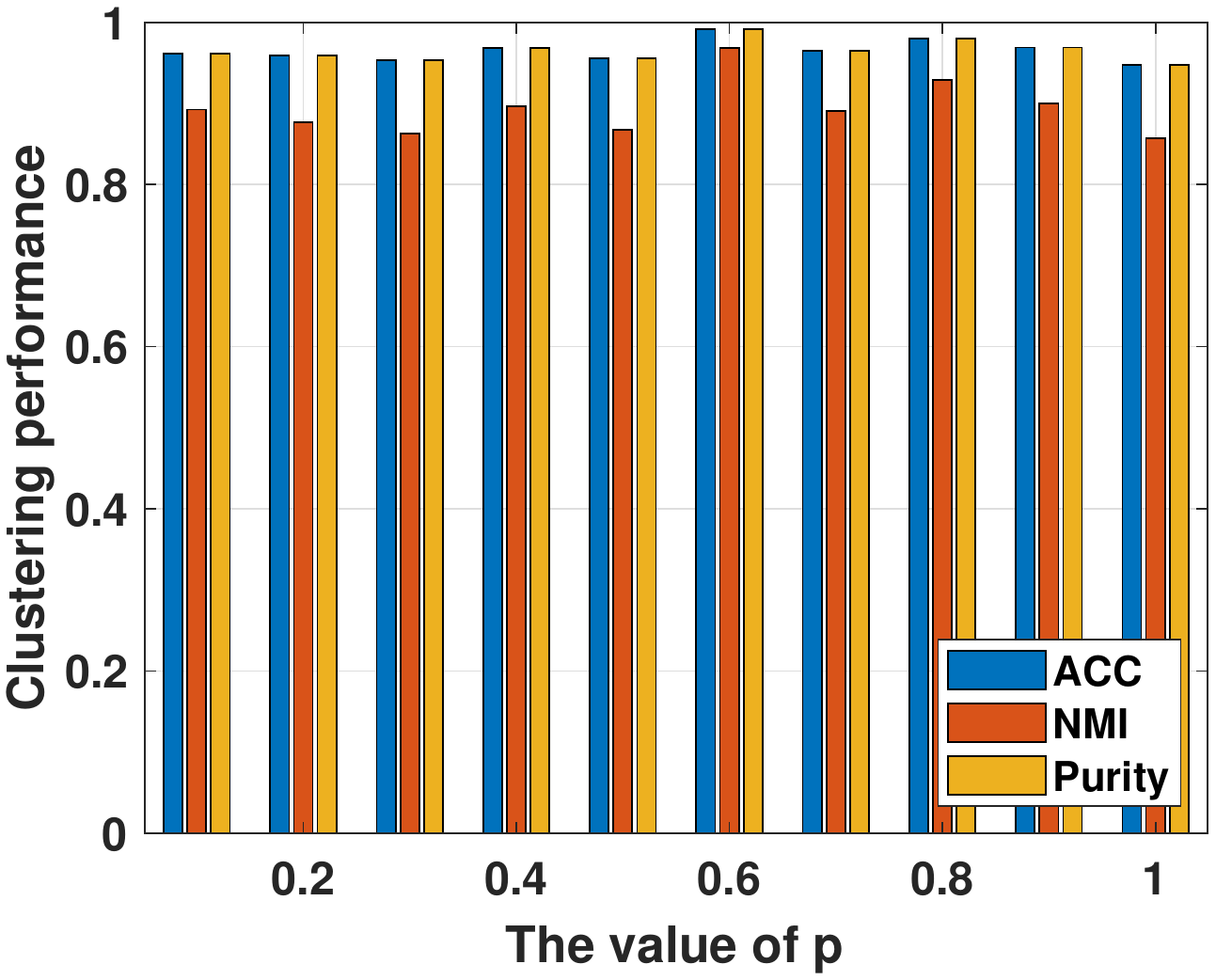}
		\end{minipage}
	}
	\centering
	\caption{Clustering performance vs. $p$ on four datasets.}
	\label{result_P}
\end{figure}

\begin{figure}[!t]
	\centering
	\subfigure[MSRC]{
		\includegraphics[width=0.47\linewidth]{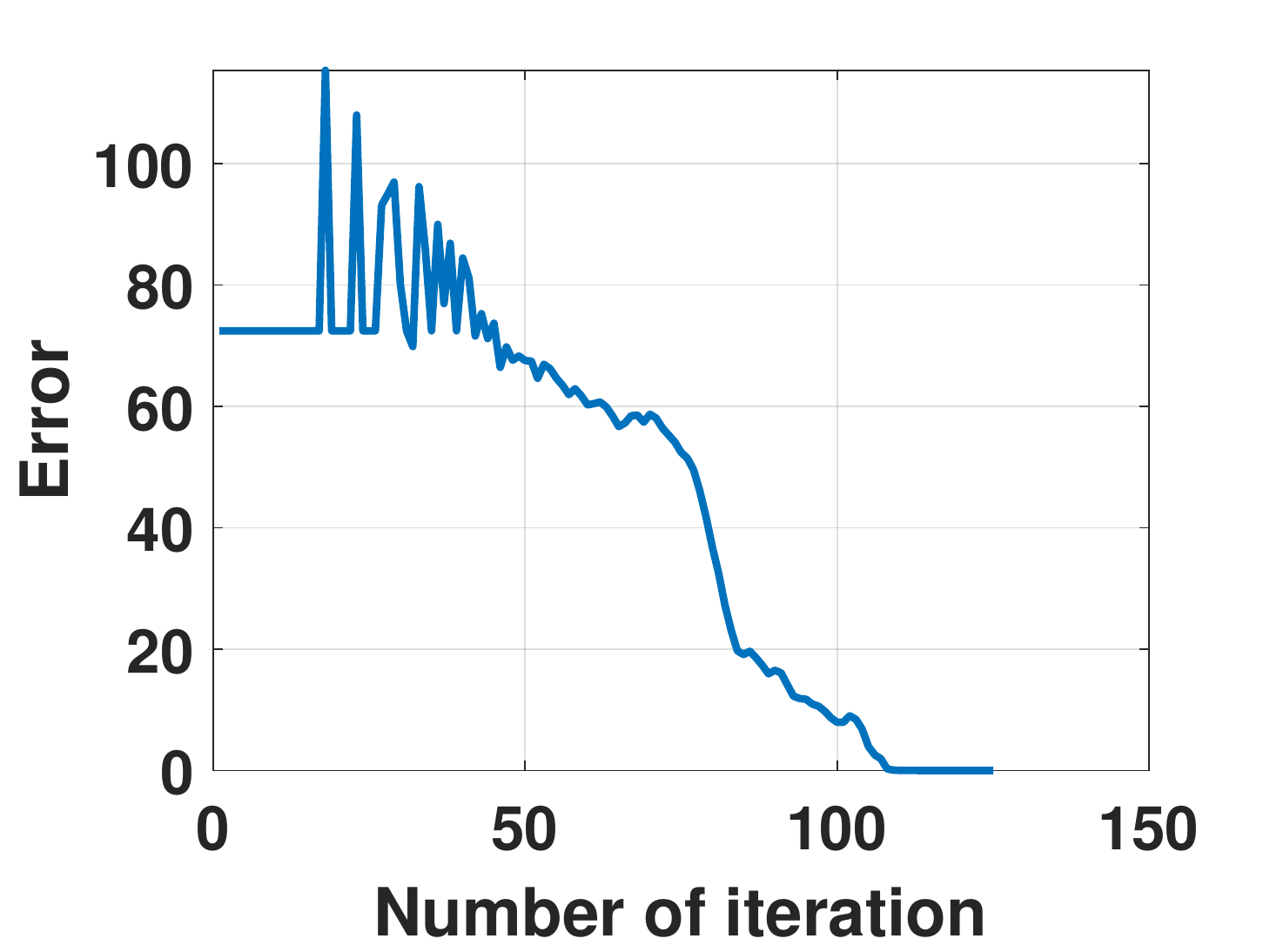}
	}
	\subfigure[ORL]{
		\includegraphics[width=0.47\linewidth]{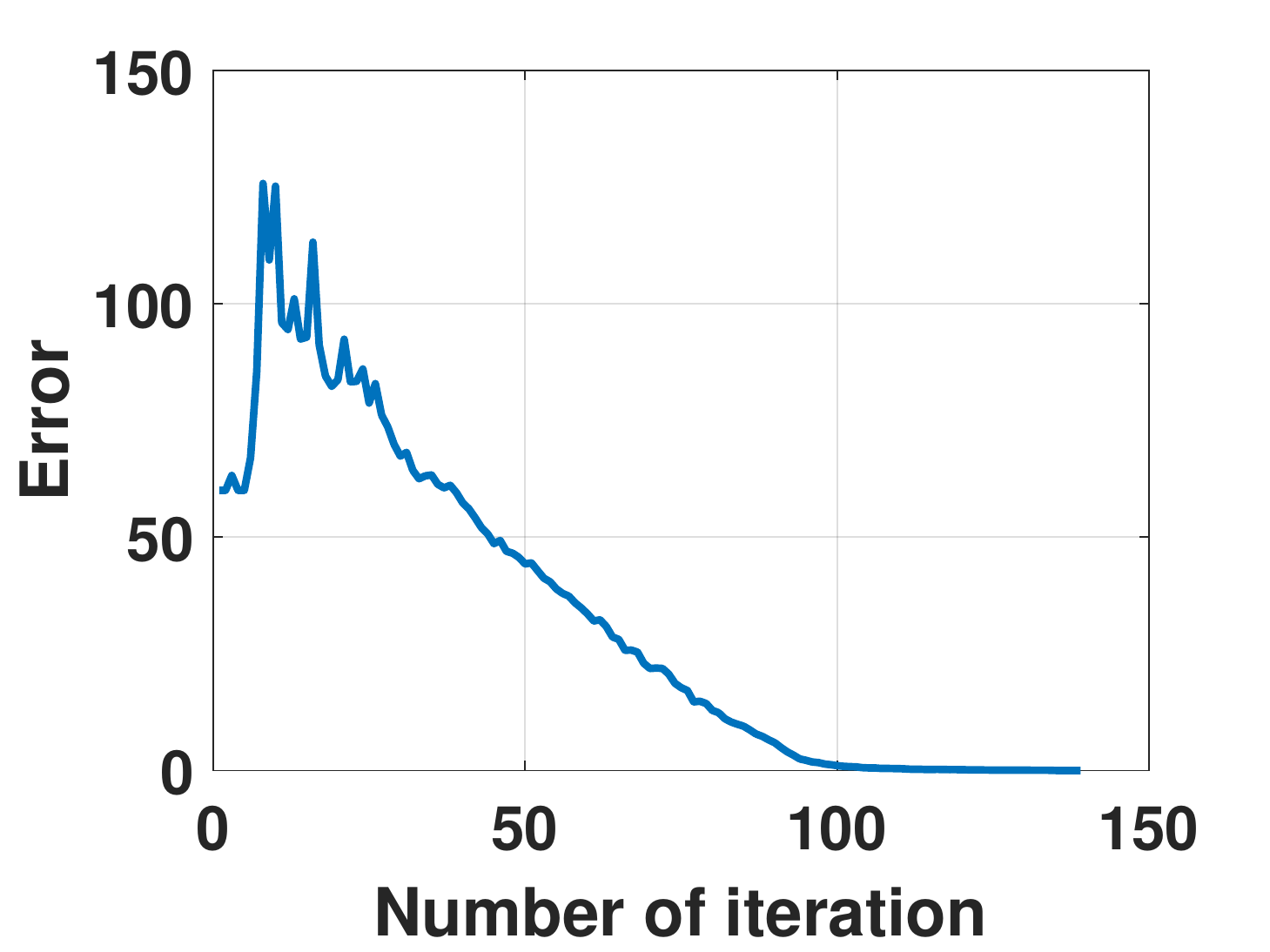}
	}
	\caption{The values of error with iterations on MSRC and ORL datasets }
	\label{result_iter}
\end{figure}

\subsubsection{Effect of parameter $r$}
In (\ref{alpha_opt}), if $r\rightarrow \infty$, the weight assigned to each view will be equal, and if $r\rightarrow 1$, the weight of the view with the minimum value of $\mathbf{M}_v$ will be 1, and others will be 0. The strategy of using $r$ not only avoids the trivial solution to the weight distribution of the different views but also controls the whole weights by exploiting one parameter.

Fig.~\ref{result_R} reveals the performances of our method on MSRC, ORL, HW, and Mnist4 datasets by showing three metrics (ACC, NMI, Purity) at different $r$, and we set $r$ from $3$ to $10$. It can be observed that when $r$ is in the range of $3$ to $9$, the clustering performance changes little on the four datasets, and the results on MSRC and HW datasets decrease when $r$ is equal to 10. This confirms that parameter $r$ plays an important role in our model.

\subsubsection{Effect of parameter $p$}
Taking MSRC, ORL, HW, and Mnist4 datasets as examples, we analyze the effect of $p$ on clustering. Specifically, we change $p$ from 0.1 to 1 with the interval of 0.1, then we report the ACC, NMI, and Purity in Fig.~\ref{result_P}. It is observed that the results under different $p$ are distinguished on HW and Mnist4 datasets, This demonstrates that the addition of $p$ has an effect on improving the clustering performance.

\begin{figure}[!t]
	\subfigure[NUS-WIDE]{
		\includegraphics[width=0.47\linewidth]{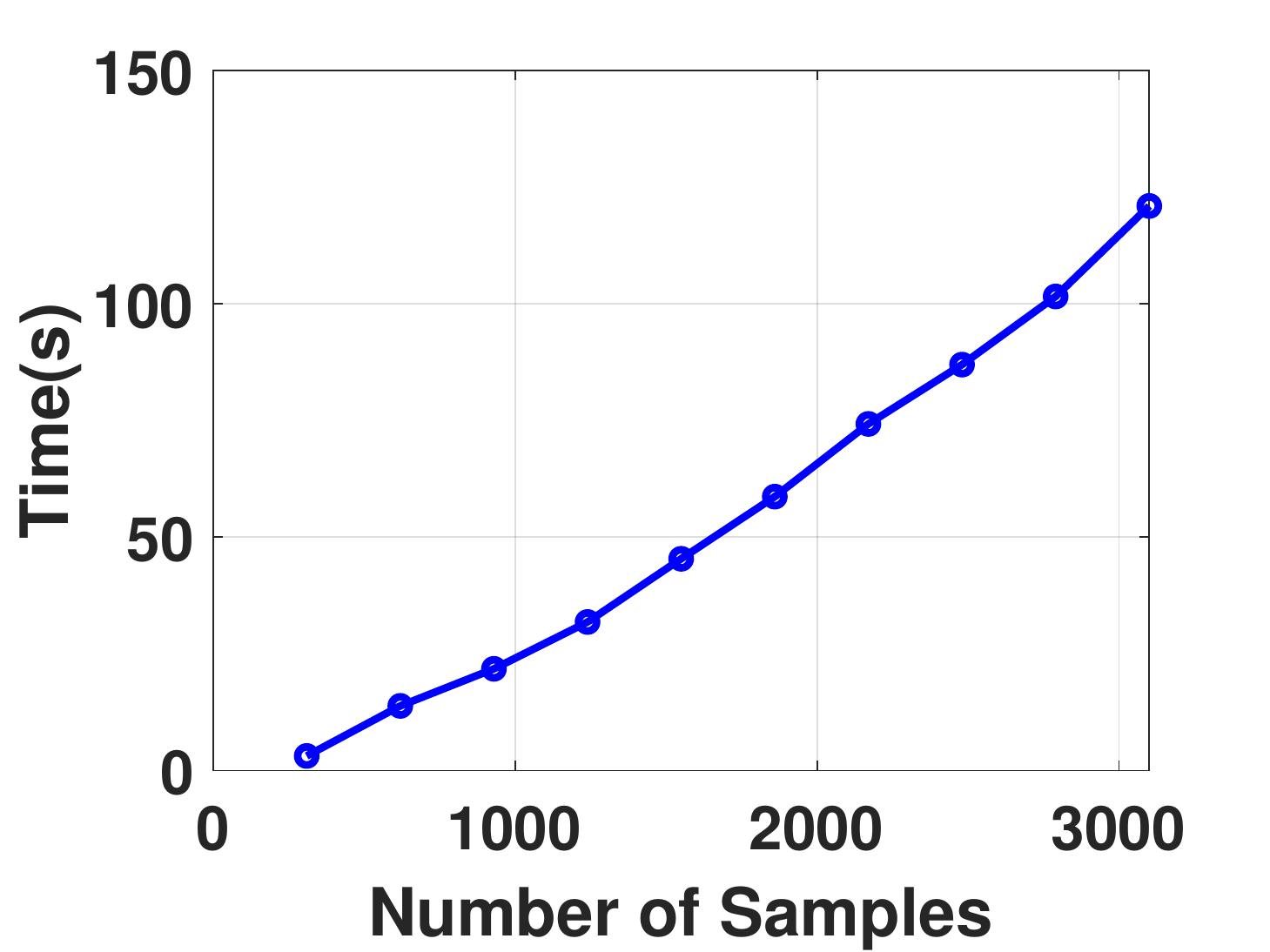}
	}
	\subfigure[Reuters]{
		\includegraphics[width=0.47\linewidth]{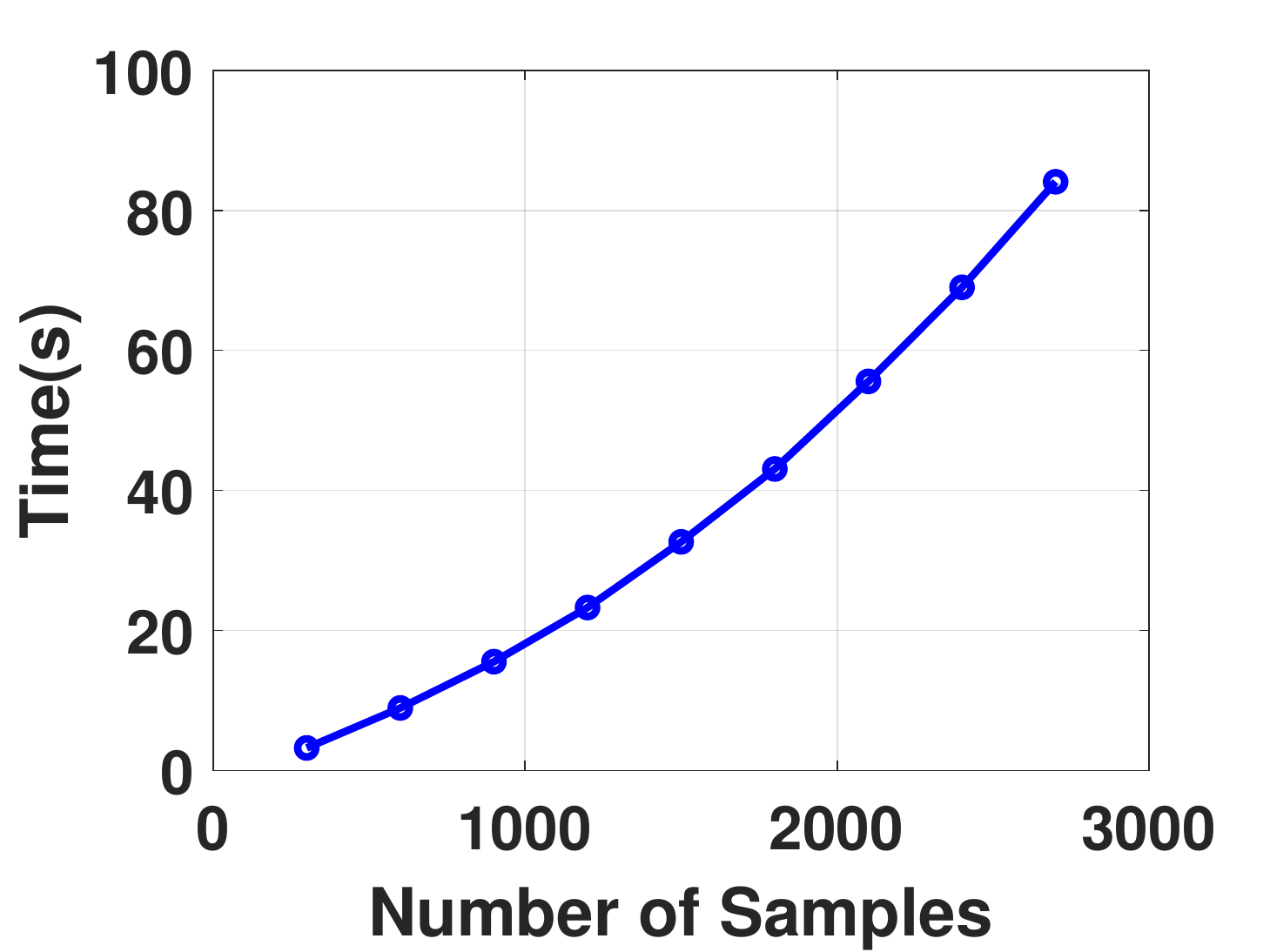}
	}
	\centering
	\caption{Running time vs. number of samples for the same number of iterations on NUS-WIDE and Reuters Datasets}
	\label{linear_verify}
\end{figure}

\subsubsection{Complexity Verification}
On the NUS-WIDE and Reuters datasets, we change the number of samples while keeping the number of iterations constant and recorded the time taken to run the MATLAB code, as shown in Figure~\ref{linear_verify}. It can be observed that the running time is approximately proportional to the number of samples, providing evidence that the computational complexity of our model scales linearly with N.

\subsubsection{Convergence}
The convergence of our method on the MSRC and ORL datasets are shown in this section. We calculate the error of $\mathbf{J}^{(v)}$ and $\mathbf{Y}^{(v)}$ from Algorithm~\ref{A1} at each number of iterations by the Frobenius norm, i.e. $Error = \sum_{i=1}^{V} {\|\mathbf{J}^{(v)} - \mathbf{Y}^{(v)} \|_F^2 }$, and record the results in Figure~\ref{result_iter}. It is observed that the value of the error decreases as the number of iterations increases and approaches 0 in less than 150 iterations. This further indicates that our proposed algorithm can converge in practical applications.

\section{Conclusion}\label{conclusion}
In this paper, we rethink \emph{k}-means from manifold learning perspective and propose a novel distance-based discrete multi-view clustering which directly achieves clusters of data without centroid estimation. This builds a bridge between manifold learning and clustering. To improve clustering performance, we employ anchor graph and Butterworth filter to construct distance matrix. It helps increase the distance between similar and dissimilar samples, while reducing the distance between similar samples. To well exploit the complementary information embedded in different views, we leverage the tensor Schatten \emph{p}-norm regularization on the 3rd-order tensor which consists of indicator matrices of different views. Extensive experimental results indicate the superiority of our proposed method.


\ifCLASSOPTIONcompsoc
\else
\fi
\ifCLASSOPTIONcaptionsoff
  \newpage
\fi

{\small
\bibliographystyle{IEEEtran}
\bibliography{egbib}
}

\end{document}